%% file: main_arxiv.tex
\documentclass[twoside]{article}

\usepackage[accepted]{aistats2022_arxiv}
\usepackage[round]{natbib}

\input{macros}

\begin{document}



\twocolumn[

\aistatstitle{Privacy Amplification by Decentralization}

\aistatsauthor{Edwige Cyffers \And Aurélien Bellet}

\aistatsaddress{Université de Lille, Inria, CNRS\\Centrale Lille,
UMR 9189 - CRIStAL\\F-59000 Lille, France \And Inria, Université de Lille,
CNRS\\ Centrale Lille, UMR 9189 - CRIStAL\\ F-59000 Lille, France} ]

\input{abstract}

\input{intro}

\input{setting}

\input{ring}

\input{complete}

\input{expes}

\input{conclu}

\input{ack}

\bibliography{main}
\bibliographystyle{apalike}

\input{appendix}

\end{document}

%% file: macros.tex
\usepackage[utf8]{inputenc} % allow utf-8 input
\usepackage[T1]{fontenc}    % use 8-bit T1 fonts
\usepackage{hyperref}       % hyperlinks
\usepackage{url}            % simple URL typesetting
\usepackage{booktabs}       % professional-quality tables
\usepackage{amsfonts}       % blackboard math symbols
\usepackage{nicefrac}       % compact symbols for 1/2, etc.
\usepackage{microtype}      % microtypography
\usepackage{xcolor}         % colors
\usepackage{mathtools}
\usepackage{amsthm}

\usepackage{algorithm}
\usepackage[noend]{algorithmic}

\usepackage{subcaption}

\definecolor{mygreen}{rgb}{0,0.5,0.0}
\definecolor{myblue}{rgb}{0,0.3,0.6}

\DeclareMathOperator*{\argmin}{arg\,min}

\newtheorem{defn}{Definition}
\newtheorem{thm}{Theorem}
\newtheorem{rmk}{Remark}
\newtheorem{prop}{Proposition}
\newtheorem{lem}{Lemma}

\newcommand{\eps}{\varepsilon}
\newcommand{\Pro}{\mathbb{P}}

\newcommand{\A}{\mathcal{A}}
\newcommand{\Obs}[1]{\mathcal{O}_{#1}}
\newcommand{\set}[1]{\left\{#1\right\}}
\newcommand{\R}{\mathbb{R}}
\newcommand{\E}{\mathbb{E}}

\DeclarePairedDelimiter\abs{\lvert}{\rvert}%

%% file: abstract.tex
% !TEX root = main_supp.tex

\begin{abstract}
Analyzing data owned by several parties while achieving a good trade-off
between utility and privacy is a key challenge in federated learning and
analytics. 
In this work, we introduce a novel relaxation of local differential privacy 
(LDP) that naturally arises in fully decentralized algorithms, i.e., when
participants exchange information by communicating along the edges of a
network graph without central coordinator. This relaxation, that
we call network DP, captures the fact
that users have only a local view of the system. To show the
relevance of network DP, we study a decentralized model of computation where a
token performs a walk on the network
graph and is updated sequentially by the party who receives it. For tasks
such as real summation, histogram computation and optimization with gradient
descent, we propose simple algorithms on ring and complete topologies. We
prove that the privacy-utility trade-offs of our algorithms under network DP
significantly improve upon what is achievable under LDP, and often
match the utility of the trusted curator model. Our results show for
the first
time that formal privacy gains can be obtained from full decentralization.
% methods based on trusted/secure
% aggregation and
% shuffling.
We also provide experiments to illustrate the improved utility of our
approach for decentralized training with stochastic gradient descent.
\end{abstract}

%% file: intro.tex
% !TEX root = main_supp.tex

\section{INTRODUCTION}

With growing public awareness and regulations on data privacy, machine
learning and data analytics are starting to transition from the classic
centralized approach, where a ``curator'' is trusted to store and analyze raw
data, to more decentralized paradigms. This shift is illustrated by the
rise of federated learning (FL) \citep{kairouz2019advances}, in which each
data subject/provider keeps her/his own data and only
shares results of local computations. In this
work, we are interested specifically in \emph{fully
decentralized} FL algorithms that do not require a central coordinator and
instead rely on peer-to-peer exchanges along
edges of a network graph, see e.g. 
\citep{Lian2017b,Colin2016a,Vanhaesebrouck2017a,Lian2018,tang18a,Bellet2018a,neglia2020,koloskova2020unified} for recent work and \citep[][Section 2.1 therein]{kairouz2019advances} for an overview. Fully decentralized approaches are usually motivated by efficiency and
scalability concerns: while a central coordinator can become a
bottleneck when dealing with the large number of participants commonly seen in
``cross-device'' applications \citep{kairouz2019advances}, in fully
decentralized methods each participant can communicate with only a
small number of peers at each step 
\citep{Lian2017b,Lian2018,neglia19infocom,neglia2020}.

In many applications involving personal or otherwise confidential information,
the participants want to keep their raw data private from other parties
involved in the FL process. Unfortunately, it is
by now well documented that the results of local computations (such as the
parameters of a machine learning model) can leak a lot of information about
the data \citep{Shokri2017}. In fact,
FL provides an additional attack surface as the participants
share intermediate updates \citep{Shokri2019,inverting}.
To control the privacy leakage, the prominent approach is based on the
standard notion of Differential
Privacy (DP) \citep{DP2006}. DP typically requires to randomly perturb
the results of computations before sharing them. This leads to a trade-off
between privacy and utility which is ruled by the magnitude of the random
perturbations.

\textbf{Related work.}
Several trust models can be considered in FL, leading
to different privacy-utility trade-offs. The strongest model is local
differential privacy (LDP) \citep{Kasiviswanathan2008,d13}, where each
participant (user) does not trust anyone and aims to protect against an
adversary that can observe everything that she/he shares.
In LDP, random perturbations are performed locally by each user, making it
convenient to design private versions of fully decentralized algorithms
in this model
% LDP can be straightforwardly applied in the context of fully
% decentralized algorithms
\citep[see e.g.,][]
{Huang2015a,Bellet2018a,onlinedecdp,leasgd,admm,adpsgd}.
Unfortunately, LDP comes at a great cost in utility: for real summation with
$n$ users, the best
possible error under LDP is a factor $\sqrt{n}$ larger
than in the centralized (trusted curator) model of DP \citep{Chan2012}. The
fundamental limits of
machine learning under LDP have been studied by \citet{Zheng2017}
and \citet{Wang2018b}.

The limitations of LDP have motivated the study of intermediate
trust models, where LDP is relaxed so as to obtain better utility while still
avoiding the need for a trusted curator. A popular approach is to resort to
cryptographic primitives to securely aggregate user contributions 
\citep{Dwork2006ourselves,Shi2011,Bonawitz2017a,ChanSS12,Jayaraman2018,Bell2020,gopa} or to securely shuffle the set of user messages so as to hide their source \citep{Cheu2019,amp_shuffling,Balle2019,Balle2019b,Ghazi2020,clones}.
% Recent work has also considered the so-called shuffle model of DP 
% \cite{Cheu2019,amp_shuffling,Balle2019,Balle2019b,Ghazi2020,clones}, where
% users send their contribution to a trusted/secure shuffler which
% permutes the set of messages so as to hide their source.
At the cost of additional computation/communication overhead, these
relaxations can provably lead to significant improvements in the
privacy-utility trade-off (sometimes matching the trusted curator model).
% their practical implementation poses
% important challenges (especially for large numbers of users). More
% importantly, they require that all users do not
However, they require all users to interact with each other at each step
and/or rely on a central coordinator. % that communicates with all users. 
% Integrating such
% solutions in fully decentralized algorithms would thus destroy the benefits of
% full decentralization.
These solutions thus appear to be incompatible with full decentralization.
%, or lead to negligible gains in utility .
% We are not aware of any work advocating this.
% integrate easily with fully decentralized algorithms that involve
% exchanges between small subsets of users (i.e., neighbors in the network
% graph) that are performed in a potentially asynchronous fashion.

A related line of work has studied mechanisms that
``amplify'' the DP guarantees of a private algorithm.
Beyond privacy amplification by shuffling 
\citep{amp_shuffling,Balle2019,clones} (based on the shuffling
primitive mentioned above), we can mention
amplification by subsampling \citep{Balle_subsampling} and amplification by
iteration \citep{amp_iter}. These schemes are generally
difficult to leverage in a federated/decentralized setting: the former
requires that the identity of subsampled participants remain secret, while
the latter assumes that only the final result is revealed.

\textbf{Our contributions.}
In this work, we propose a \emph{novel relaxation of LDP where
users have only a local view of the decentralized system}, which is a natural assumption in fully
decentralized settings. This relaxation, called 
\emph{network differential privacy}, effectively
captures the fact that each user only observes
information received from her/his
neighbors in the network graph. Network DP can also account for potential
collusions
between users.
We initiate the study
of algorithms under network DP in a
decentralized model of computation where a
token containing the current estimate performs a walk on
the network graph and is updated sequentially by the user who receives
it. This model has been studied in previous work as a way to perform
(non-private) decentralized estimation and optimization with less
communication and computation overhead than algorithms that require all users
to communicate with their neighbors at each step 
\citep{Nedic2009,Johansson2009,Walkman,Ayache2020}. 

We start by analyzing the case of a (deterministic) walk over a directed
ring for the tasks of computing real summations and discrete histograms. In
both cases, we propose simple algorithms which achieve a privacy gain
of $O(1/\sqrt{n})$ compared
to LDP, thereby \emph{matching the privacy-utility trade-off of a trusted
aggregator} without relying on any costly secure multi-party computation
protocol.
% based on secure aggregation
% and secure shuffling while requiring only on a small number of secure
% communication channels.
Noting that the ring topology is not very
robust to collusions, we then consider the case of random walks
over a complete graph. We provide an algorithm for real summation and prove a privacy amplification
result of $O(1/\sqrt{n})$ compared to the same algorithm analyzed under LDP,
again matching the privacy-utility trade-off of the trusted curator model. We
also discuss a natural extension for computing discrete histograms. Finally,
we turn to the task of optimization with stochastic gradient descent and 
propose a decentralized SGD algorithm that achieves a privacy
amplification of $O
(\ln n/\sqrt{n})$ in some regimes, nearly matching the utility of 
\emph{centralized} differentially
private SGD \citep{Bassily2014a}.
Interestingly, the above algorithms can tolerate
a constant number of collusions at the cost of some reduction in the privacy
amplification effect.
At the technical level, our theoretical analysis leverages recent results
on privacy amplification
by subsampling \citep{Balle_subsampling}, shuffling 
\citep{amp_shuffling,Balle2019,clones}
and iteration \citep{amp_iter} in a novel decentralized context: this is
made possible by the restricted view of participants offered by decentralized
algorithms and adequately captured by our notion of
network DP.
At the empirical level, we show through experiments that
privacy gains are significant in practice both for simple analytics
and for training models in federated learning scenarios.

To the best of our knowledge, our work is the first to show that \emph{formal
privacy gains can be naturally obtained from full decentralization} (i.e.,
from having no central coordinator).
% i.e.,
% when users
% communicate in a peer-to-peer fashion instead of relying on a central 
% (untrusted) aggregator for all communications.
Our results imply that the true privacy
guarantees of some fully decentralized algorithms have been largely
underestimated,
providing a new incentive for using such approaches beyond the
usual motivation of scalability. We believe that our work opens several
promising perspectives, which we outline in the conclusion.

\textbf{Paper outline.}
The paper is organized as follows.
Section~\ref{sec:network-dp}
introduces our notion of network DP and the
decentralized model of computation that we study. Section~\ref{sec:ring}
focuses on the case of a fixed ring topology, while Section~\ref{sec:complete}
considers random walks on a complete graph. We present some numerical results
in Section~\ref{sec:exp} and draw some perspectives for future work in
Section~\ref{sec:conclu}.

%% file: setting.tex
% !TEX root = main_supp.tex

\section{NETWORK DP AND DECENTRALIZED MODEL}
\label{sec:network-dp}

% In this section, we define our main notations and introduce network
% differential privacy (network DP), our novel relaxation of LDP. Then we
% describe the family
% of decentralized protocols that we will study under network DP.

Let $V= \{1,\dots,n\}$ be a set of $n$ users (or parties), which are assumed
to be honest-but-curious (i.e., they truthfully follow the protocol). Each
user $u$ holds a private dataset $D_u$, which we keep abstract at this point.
We denote by $D=D_1\cup\dots\cup D_n$
the union of all user datasets, and by $D \sim_u
D'$ the fact that datasets $D$ and $D'$ of same size differ only on user $u$'s
data. This defines a \emph{neighboring relation} over datasets which is
sometimes referred to as user-level DP \citep{user-dp}. This relation is
weaker than the one used in classic DP and will thus provide stronger privacy
guarantees. Indeed, it seeks to hide the influence of a \emph{user's whole
dataset} rather than
a single of its data points.

We consider a fully decentralized
setting, in which users are nodes in a network graph $G=(V,E)$ and
an edge $(u,v)\in E$ indicates that user $u$ can send messages to user $v$.
The graph may be directed or undirected, and could in principle change over
time although we will restrict our attention to fixed topologies.
For the purpose of quantifying privacy guarantees, a
decentralized algorithm $\A$ will be viewed as a (randomized) mapping which
takes as input a dataset $D$
and
outputs the transcript of all messages exchanged between users
over the network. We denote the (random) output in an abstract manner by
% \begin{equation}
% \label{eq:transcript}
$\mathcal{A}(D) = ((u,m,v) : \text{ user } u \text{ sent message with
content } m \text{ to user } v)$.
% \end{equation}

\textbf{Network DP.}
The key idea of our new relaxation of LDP is to consider that \emph{a given
user does not~have access to
the full transcript $\A(D)$ but only to the messages she/he is involved in} 
(this can be enforced by the use of secure communication channels). We denote
the corresponding view of a user
$u$ by
\begin{equation}
\label{eq:obs_abstract}
\Obs{u}(\A(D))= ((v,m,v')\in\mathcal{A}(D) : v=u \text{ or } v'=u).
\end{equation}

\begin{defn}[Network Differential Privacy]
  \label{def:network_dp}
  An algorithm $\A$ satisfies $(\eps, \delta)$-network DP if for all pairs of
  distinct users $u, v\in V$ and all
  pairs   of neighboring datasets $D \sim_u D'$, we have:
  \begin{equation}
  \label{eq:network-dp}
  \Pro(\Obs{v}(\A(D))) \leq e^{\eps} \Pro(\Obs{v}(\A(D')))+
  \delta.
  \end{equation}
\end{defn}
Network DP essentially requires that for any two users $u$ and $v$, the
information gathered by user $v$ during the execution of $\A$ should not depend
too much on user $u$'s data. Network DP can be thought of as analyzing the
composition of the operator $\Obs{v}$ with the algorithm $\A$. The hope is
that in some cases $\Obs{v}\circ\A$ is more private than $\A$: in
other words, that applying $\Obs
{v}$
\emph{amplifies} the privacy guarantees of $\A$. Note
that if $\Obs{v}$ is the identity map
(i.e., if each user is able to observe all messages), then
Eq.~\ref{eq:network-dp} boils down to
local DP.
% Network DP boils down to classic centralized DP if
% each node has exactly one piece of data and if the specific case of a node $v$ has access to all the outputs generated by $\A$ on the dataset. For a fixed node, the local differential privacy can be interpreted as the network differential privacy where all its information transit by another node without being mixed with other values.

We can naturally extend Definition~\ref{def:network_dp} to account for
potential \emph{collusions} between users. As common in the literature,
we assume an upper
bound $c$ on the number of users that can possibly
collude. The identity of colluders is however
unknown to other users. In this setting, we would like to be private
with respect to the
aggregated information $\Obs{V'}=\cup_{v\in V'}\Obs{v}$ acquired by any possible
subset $V'$ of $c$ users, as captured by the following
generalization of Definition~\ref{def:network_dp}.

  \begin{defn}[Network DP with collusions]
  \label{def:network_dp_col}
  An algorithm $\A$ is $(c,\eps, \delta)$-network DP if for each user $u$, all
   subsets $V' \subset V$ such that $\abs{V'} \leq c$, and all pairs of
   neighboring datasets $D \sim_u D'$, we
   have:
     \begin{equation}
  \label{eq:network-dp-col}
  \Pro(\Obs{V'}(\A(D)) \leq e^{\eps} \Pro(\Obs{V'}(\A(D')) + \delta.
  \end{equation}
  \end{defn}

\textbf{Decentralized computation model.}
In this work, we study network DP for decentralized algorithms that perform
computations via sequential updates to a \emph{token} $\tau$ walking through
the nodes by following the edges of the graph $G$.
At each step, the token $\tau$ resides at some node $u$ and is updated by
\begin{equation}
\label{eq:token_update}
\tau \leftarrow \tau + x_u^k,\quad\text{with } x_u^k=g^k(\tau;D_u),
\end{equation}
where $x_u^k=g^k(\tau;D_u)$ denotes the
contribution of user $u$. The notation highlights the fact that this
contribution may depend on the current value $\tau$ of the token
as well as on the number of times $k$ that the token visited $u$ so far. The
token $\tau$ is then sent to another user $v$ for which $(u,v)\in E$.

Provided that the walk follows some properties (e.g., corresponds
to a deterministic cycle or a random walk that is suitably ergodic), this
model of computation allows to optimize sums of local cost functions using 
(stochastic) gradient descent
\citep{Nedic2009,Johansson2009,Walkman,Ayache2020} (sometimes referred to as
incremental gradient methods) and hence to train
machine learning models. In this case, the token $\tau$ holds the model
parameters and $x_u^k$ is a (stochastic) gradient of the local loss function
of user $u$ evaluated at $\tau$.
Such decentralized algorithms can also be used
to compute summaries of the users' data, for instance any commutative
and associative operation like sums/averages and discrete histograms.
In these cases, the contributions of a given user may correspond to different values acquired over time, such as
power consumption in smart metering or item ratings in collaborative
filtering applications.

%% file: ring.tex
% !TEX root = main_supp.tex

\section{WALK ON A RING}
\label{sec:ring}

In this section, we start by analyzing a simple special case where the graph
is a directed ring, i.e., $E = \{(u,u+1)\}_{u=1}^{n-1} \cup \{(n, 1)\}$. The
token starts at user $1$ and goes through the ring $K$ times. The ring (i.e.,
ordering of the nodes) is assumed to be public.

\subsection{Real Summation}
\label{sec:ring_sum}
We first consider the
task of estimating the sum $\bar{x} = \sum_{u=1}^n \sum_{k=1}^K x_u^k$ where
the $x$'s are bounded real numbers and $x_u^k$ represents the contribution of
user $u$ at round $k$.
For this problem, the standard approach in local DP is to
add random noise to each single contribution before releasing it. For
generality, we consider an abstract mechanism $\text{Perturb}(x;\sigma)$
which adds centered noise with standard deviation $\sigma$ to the contribution
$x$ (e.g., the Gaussian or Laplace mechanism). Let $\sigma_
{loc}$ be the standard deviation of the noise required so that
$\text{Perturb}(\cdot;\sigma_{loc})$ satisfies $(\eps,\delta)$-LDP.

Consider now the simple decentralized protocol in
Algorithm~\ref{algo:real_ring}, where noise with the same standard
deviation $\sigma_{loc}$ is added \emph{only once every $n-1$ hops of the token}.
By leveraging the fact that the view of each user $u$ is restricted to the
values taken by the token at each of its $K$ visits to $u$, combined with
advanced composition \citep{boosting}, we have the
following result
(see Appendix~\ref{ringproof} for the proof).
% (see supplementary material for the proof).

        \begin{algorithm}[t]
        \centering
        \caption{Private real summation on the ring.}
        \label{algo:real_ring}
        \begin{algorithmic}[1]
        \STATE $\tau \leftarrow 0$; $a \leftarrow 0$\;
        \FOR{$k = 1$ to $K$}
          \FOR{$u=1$ to $n$}
            \IF{$a = 0$}
              \STATE $\tau \leftarrow \tau + \text{Perturb}(x_u^k;\sigma_
              {loc})$\;
              \STATE $a = n-2$\;
            \ELSE
            \STATE $\tau \leftarrow \tau + x_u^k$; $a \leftarrow a-1$\;
            \ENDIF
            \ENDFOR
            \ENDFOR
        \STATE \textbf{return} $\tau$
        \end{algorithmic}
        \end{algorithm}

    \begin{thm}
    \label{thm:ring-sum}
    Let $\eps,\delta>0$. Algorithm~\ref{algo:real_ring} outputs an
    unbiased estimate of $\bar{x}$ with standard deviation
    $
    \sqrt{\lfloor K n/(n-1)\rfloor} \sigma_{loc}$, and
    is $(\sqrt{2 K \ln (1/\delta')}\eps+ K \eps(e^{\eps}-1),K  \delta +
    \delta')$-network DP for any $\delta'>0$.
    \end{thm}

To match the same privacy guarantees, LDP incurs a standard deviation of $
\sqrt{K n} \sigma_{loc}$. Therefore, Algorithm~\ref{algo:real_ring} provides an
$O(1/\sqrt{n})$ reduction in error or, equivalently, an $O
(1/\sqrt{n})$ gain in $\eps$. In fact,
Algorithm~\ref{algo:real_ring}
achieves the same privacy-utility trade-off as a \emph{trusted
central aggregator} that would iteratively aggregate the raw contributions of
all users at each
round $k$ and perturb the result before sending it to the users, as
done in federated
learning algorithms with a trusted server \citep{kairouz2019advances}.
% Algorithm~\ref{algo:real_ring} nearly
% achieves the same privacy-utility trade-off as a \emph{trusted curator} who
% would collect raw contributions and add noise with standard deviation
% $\sqrt{K}\sigma_{loc}$ to true sum $\bar{x}$.

\begin{rmk}
\label{rm:distributed-noise}
% In Algorithm~\ref{algo:real_ring}, a single user is responsible for
% adding the necessary noise in each cycle. Alternatively,
We can
design variants of Algorithm~\ref{algo:real_ring} in
which noise addition is distributed across users. Using the Gaussian
mechanism, each user can add noise with std. dev.
$\sigma'_{loc} =
\sigma_{loc}/\sqrt{n}$, except for the very first
contribution which requires std. dev. $\sigma_{loc}$ to
properly hide the contributions of users in the first cycle. 
The total added noise has std. dev. $\sqrt{\lfloor K n/(n-1)\rfloor
+1} \sigma_{loc}$, leading to
same utility as Algorithm~\ref{algo:real_ring} (up to a constant factor
that is negligible when $K$ is large).
\end{rmk}

\subsection{Discrete Histogram Computation}
\label{sec:hist_ring}

We now turn to the computation of histograms over
a discrete domain $[L]=\{1,\dots,L\}$. The goal
is to compute $h\in\mathbb{N}^L$ s.t. $h_l=\sum_{u=1}^n\sum_{k=1}^K
\mathbb{I}[x_u^k = l]$, where $x_u^k\in[L]$.
A classic approach in LDP is based on
$L$-ary randomized response \citep{kRR}, where each user submits its
true value with probability $1-\gamma$ and a uniformly random value with
probability $\gamma$. We denote this primitive by $RR_{\gamma}:
[L]\rightarrow[L]$.

In our setting with a ring network, we propose
Algorithm~\ref{algo:discrete_ring}, where each contribution of a user is
randomized using $RR_{\gamma}$ before being added to the token $\tau\in
\mathbb{N}^L$. Additionally, $\tau$
is initialized
with enough random elements to hide the first contributions.
% Similar behaviors arise for histogram
% computation, where at each round, each node commits in one option above $L$ possible ones.
Note that at each step, the token contains a partial histogram equivalent to a
shuffling of the contributions added so far, allowing us to leverage
results on \emph{privacy
amplification by shuffling} \citep{amp_shuffling,Balle2019,clones}.
In particular, we can prove the following utility and privacy
guarantees for Algorithm~\ref{algo:discrete_ring}
(see Appendix~\ref{histoproof} for the proof).
% (see supplementary material for the proof).

\begin{algorithm}[t]
        \centering
        \caption{Private histogram on the ring.}
        \label{algo:discrete_ring}
        \begin{algorithmic}[1]      
        % $\gamma \leftarrow f(\eps, L)$\;
        % $\tau \leftarrow \mathcal{U} ( \{\tau \in \mathbb{N}^L, \sum_{i=0}^{L-1} \tau_i = \lceil \frac{L}{ e^{12\eps_0 \sqrt{\frac{\log(1/\delta)}{n}}}+ L -1} \rceil  \} )$\;
        \STATE Init. $\tau\in\mathbb{N}^L$ with $\gamma n$ random
        elements\;
        \FOR{$k = 1$ to $K$}
          \FOR{$u=1$ to $n$}
            \STATE $y_u^k \leftarrow RR_{\gamma}(x_u^k)$\;
            \STATE $\tau[y_u^k] \leftarrow \tau[y_u^k] + 1$\; 
        \ENDFOR
        \ENDFOR
        \FOR{$i=0$ to $L-1$}
          \STATE $\tau[i] \leftarrow \frac{\tau[i] - \gamma/L}{1-\gamma}$\;
        \ENDFOR
        \STATE \textbf{return} $\tau$
        \end{algorithmic}        
        \end{algorithm}

% To ensuring differential
% privacy
% in this discrete context, the classical method is randomized response, where each node answers truthfully with some probability, and uniformly at random otherwise. In this setting, we might also chose which nodes answers randomly at each rounds, to ensure an exact number of random response.

    \begin{thm}
    \label{thm:ring-histogram}
    Let $\epsilon<\frac{1}{2}$, $\delta\in(0,\frac{1}{100})$, and $n>1000$. Let $\gamma = 
    L/(\exp(12\eps \sqrt{\log(1/\delta)/n})+ L -1)$.
    Algorithm~\ref{algo:discrete_ring} outputs an
    unbiased estimate of the histogram with $\gamma n(K+1)$ expected random
    responses. Furthermore, it satisfies $(\sqrt{2 K
    \ln (1/\delta')}\eps+ K \eps(e^{\eps}-1),K  \delta +
    \delta')$-network DP for any $\delta'>0$.
    \end{thm}

Achieving the same privacy in LDP would require $\gamma$ to be
constant in $n$, hence $\sqrt{n}$ times more random responses.
Equivalently, if we fix utility (i.e., $\gamma$),
Theorem~\ref{thm:ring-histogram}
shows that Algorithm~\ref{algo:discrete_ring} again provides a
privacy gain of $\frac{1}{n} \sqrt{ n / \ln(1/\delta)}=O
(1/\sqrt{n})$ compared to LDP.

\begin{rmk}
For clarity, Theorem~\ref{thm:ring-histogram} relies on the
amplification by
shuffling result of \cite{amp_shuffling} which has a simple closed form.
A tighter and more general result (with milder restrictions on the values of
$n$, $\eps$ and $\delta$) can be readily obtained by
using the results of \cite{Balle2019} and \cite{clones}. 
\end{rmk}

\begin{rmk}
Algorithm~\ref{algo:real_ring} (real summation) can also
be used to perform histogram computation. However, for domains of large
cardinality $L$ (e.g., $L\gg n$), Algorithm~\ref{algo:discrete_ring} requires
fewer random numbers and maintains a sparse (more compact)
representation
of the histogram.
\end{rmk}

\subsection{Discussion}
\label{sec:ring-discussion}

We have seen that decentralized computation over a ring provides a simple
way to achieve utility similar to a trusted aggregator thanks to
the sequential communication that hides the contribution of previous users
in a summary.
We emphasize that this is achieved without relying on a central
server (only local communications) or resorting
to costly multi-party computation protocols (only two secure
communication channels per user are needed).
% We stress the fact that secure aggregation
% and secure shuffling are non-trivial secure
% multi-party computation protocols which can pose implementation and
% scalability challenges \cite{Bonawitz2017a}. In contrast, our approach is
% very simple as it only requires to establish two secure communication channels per
% user.
Interestingly, the ring topology is often used in practical deployments and
theoretical analysis of (non-private) decentralized algorithms 
\citep{Lian2017b,tang18a,koloskova2020unified,neglia2020,marfoq}, owing to its
simplicity and good empirical performance.
Finally, we note an interesting connection between the case of network DP over
a ring topology and the pan-privacy model for streaming
algorithms \citep{pan-privacy} (see Appendix~\ref{app:pan_privacy} for details).

Despite the above advantages, the use of a fixed ring topology
has some limitations. First, the above
algorithms are not robust to collusions: in particular,
% The non-collusion
% assumption is essential in Algorithm~\ref{algo:real_ring}:
if two users collude and share
their view, Algorithm~\ref{algo:real_ring} does not satisfy DP.
% (when it is the turn
% of one of the colluding users to add the protecting noise,
% there is no protection left).
While this can be mitigated by distributing the
noise addition across users (Remark~\ref{rm:distributed-noise}),
a node placed between two colluding nodes (or with few
honest users in-between) would suffer largely degraded privacy
guarantees. A similar reasoning holds for Algorithm~\ref{algo:discrete_ring}.
Second, a fixed ring topology is not well suited to extensions to gradient
descent, where we
would like to leverage privacy amplification by iteration \citep{amp_iter}. In
the latter, the privacy guarantee for a given user (data
point) grows with the number of
gradient steps that come after it. In a fixed ring, the privacy
of a
user $u$ with respect to another user $v$ would thus depend on
their relative positions in the ring (e.g., there would be no privacy
amplification when $v$ is the user who comes immediately
after $u$).
These limitations motivate us to consider random walks on a complete graph.

%% file: complete.tex
% !TEX root = main_supp.tex

\section{WALK ON COMPLETE GRAPH}
\label{sec:complete}

In this section, we consider the case of a random walk on the complete graph.
In other words, at each step, the token is sent to a user chosen
uniformly at random among $V$. We consider random walks of fixed length $T>0$,
hence the number of times a given user contributes is itself random.
We assume the token path to be hidden, including the
previous sender and the next receiver, so the only knowledge of a user is the
content of the messages that she/he receives and sends.

\subsection{Real Summation}

For real summation, we consider the simple and natural protocol shown in
Algorithm~\ref{algo:complete}: a
user $u$
receiving the token $\tau$ for the $k$-th time updates it with $\tau
\leftarrow \tau + \text{Perturb}(x_u^k;\sigma_{loc})$. As in
Section~\ref{sec:ring_sum}, $\sigma_{loc}$ is set such that $\text{Perturb}
(\cdot;\sigma_{loc})$
satisfies $(\eps,\delta)$-LDP, and thus implicitly depends on $\eps$ and
$\delta$.
% The next theorem gives network DP guarantees, which rely on
% \emph{the intermediate aggregations of values between two visits of the token
% to a given user} and \emph{the secrecy of the path taken by the token}.
% For clarity, we state below an asymptotic result to give the main order of
% magnitude, but stress the fact that it is derived from a 
% \emph{non-asymptotic} (albeit more complex) formula given in the
% supplementary.
We now show network DP guarantees, which rely on
\emph{the intermediate aggregations of values between two visits of the token
to a given user} and \emph{the secrecy of the path taken by the token}.
For clarity, the theorem below gives only the main order of magnitude, but the
complete and tighter formula can be found in Appendix~\ref{completeproof}.

\begin{algorithm}[t]
  \centering
  \caption{Private summation on a complete graph.}
  \label{algo:complete}
  \begin{algorithmic}[1]        
  \STATE $\tau \leftarrow 0$, $k_1\leftarrow 0,\dots,k_n\leftarrow 0$\;
  \FOR{$t = 1$ to $T$}
    \STATE Draw $u \sim \mathcal{U}({1,\dots, n})$\;
    \STATE $k_u\leftarrow k_u+1$\;
    \STATE $\tau \leftarrow \tau + \text{Perturb}(x_u^{k_u};\sigma_{loc})$\;
  \ENDFOR
  \STATE \textbf{return} $\tau$
  \end{algorithmic}        
\end{algorithm}

\begin{thm}
\label{thm:complete}
Let $\eps<1$ and $\delta>0$. Algorithm~\ref{algo:complete} outputs an
unbiased
estimate of the sum of $T$ contributions with standard deviation $
\sqrt{T}\sigma_{loc}$, and satisfies $(\eps',(N_v + T/n) \delta + \delta' + 
\hat{\delta})$-network
DP for all $\delta',\hat{\delta} > 0$ with
\begin{equation}
\label{eq:completereal_bigO}
\eps' = O\Big( \sqrt{N_v\ln(1/\delta')} \eps / \sqrt{n} \Big),
\end{equation}
where $N_v = \frac{T}{n} + \sqrt{\frac{3}{n} T \ln (1/\hat{\delta})}$.
\end{thm}

\begin{proof}[Sketch of proof.]
We summarize here the main steps (see Appendix~\ref{completeproof} for details). We fix a user $v$
and quantify how much information about the private data of another user $u$ is leaked to $v$ from
the visits of the token. 
The number of visits to $v$ follows a binomial law $\mathcal{B}(T,
1/n)$ that we upper bound by $N_v$ using Chernoff with probability $1-
\hat{\delta}$.
Then, for a contribution of $u$ at time $t$, it is sufficient to consider the cycle formed by the random walk between the two successive passages in $v$ containing $t$.
To be able to use amplification by subsampling \citep{Balle_subsampling}, we
actually consider a fictive walk where
each cycle cannot exceed a length of $n$: if a cycle is larger, we assume
that the value of the token is observed by $v$ every $n$ steps. As the
information leaked to $v$ by the
actual walk can be obtained by post-processing of this fictive walk, it is
enough to compute the privacy loss of the fictive walk, which has at most
$N_v+T/n$ cycles (with high probability). Then, we prove that each cycle
incurs
at most a privacy loss of $3 \eps / \sqrt{n}$ by combining intermediate
aggregations and amplification by subsampling. We conclude with $\eps'=
O( \sqrt{(N_v + T/n)\log(1/\delta')} \frac{\varepsilon}{\sqrt{n}} )$ and
$\delta_f =
(N_v + T/n) \delta + \delta' + \hat{\delta}$ by advanced composition.
\end{proof}

The same algorithm analyzed under LDP yields $\eps'=O(\sqrt{N_v\ln
(1/\delta')}\eps)$, which is optimal for averaging $N_v$ contributions per
user in the local model. For $T = \Omega(n)$,
Theorem~\ref{thm:complete} thus shows that network DP asymptotically provides
a privacy amplification of $O(1/\sqrt{n})$ over LDP and matches the
privacy-utility trade-off of a trusted aggregator.
% We give the \emph{non-asymptotic} formula (from which 
% \eqref{eq:completereal_bigO}) in the supplementary.
 We will see in
Section~\ref{sec:exp} that our complete (tighter) formula given in
Appendix~\ref{completeproof} improves
upon local DP as soon as $n\geq20$
(Figure~\ref{fig:fixed_contrib}). We also show that the gains are
significantly stronger
in practice than what our theoretical results guarantee
(Figure~\ref{fig:gaussian}).

\textbf{Extension to discrete histogram computation.}
We can obtain a similar result for histograms
by
bounding the
privacy loss incurred for each cycle by using amplification by shuffling
\citep{amp_shuffling,Balle2019,clones}, similar to what we did for the
ring topology (Section~\ref{sec:hist_ring}). Details are in
Appendix~\ref{hist_complete}.

\subsection{Optimization with SGD}

We now turn to the task of
private convex optimization with stochastic gradient descent (SGD).
% Let $\mathcal{X}$ be some abstract data domain and assume for simplicity
% of presentation that
% each user $u$ holds a single
% data point $x_u\in\mathcal{X}$, i.e., $D_u=\{x_u\}$.
Let $\mathcal{W}
\subseteq 
\mathbb{R}^d$ be a
convex set and $f(\cdot; D_1),\dots, f(\cdot; D_n)$ be a set of convex
$L$-Lipschitz and $\beta$-smooth functions over $\mathcal{W}$ associated with
each user. We denote by
$\Pi_{\mathcal{W}}(w)=\argmin_{w'\in\mathcal{W}}\|w-w'\|$ the Euclidean
projection onto the set $\mathcal{W}$.
We aim to privately solve the following optimization problem:
\begin{equation}
\label{eq:optim}
\textstyle w^*\in \argmin_{w\in\mathcal{W}} \big\{F(w):=\frac{1}{n}\sum_
{u=1}^n f
(w;D_u)\big\}.
\end{equation}
Eq.~\ref{eq:optim} encompasses many machine learning tasks (e.g., ridge
and logistic regression, SVMs, etc).

\begin{algorithm}[t]
  \centering
  \caption{Private SGD on a complete graph.}
  \label{algo:sgd}
  \begin{algorithmic}[1]
  \STATE Initialize $\tau \in \mathcal{W}$
  \FOR{$t = 1$ to $T$}
    \STATE Draw $u \sim \mathcal{U}({1,\dots, n})$\;
    \STATE $Z=[Z_1,\dots,Z_d], \text{  } Z_i\sim \mathcal{N}\big(0,
    \frac{8L^2\ln(1.25/\delta)}{\eps^2}\big)$\;\label{line:var}
    \STATE $\tau \leftarrow \Pi_{\mathcal{W}}(\tau - \eta (\nabla_\tau f(\tau;
    D_u) +
    Z))$\;
  \ENDFOR
  \STATE \textbf{return} $\tau$
  \end{algorithmic}        
\end{algorithm}

To privately approximate $w^*$, we
propose Algorithm~\ref{algo:sgd}.
Here,
the token $\tau\in\mathcal{W}$ represents the current iterate. At each step,
the user $u$ holding the token performs a projected noisy gradient step and
sends the updated token to a random user. We
rely on
the Gaussian mechanism to ensure that the noisy 
version of the gradient $\nabla_\tau f(\tau; D_u)+Z$ satisfies $
(\eps,\delta)$-LDP: the variance $\sigma^2$ of the noise
in line \ref{line:var} of Algorithm~\ref{algo:sgd} follows from the fact that
gradients of $L$-Lipschitz functions have sensitivity bounded by $2L$ \citep{Bassily2014a}.
% We assume here that the noise follows a Gaussian distribution $Z \sim 
% \mathcal{N}(0, L^2 \sigma_{loc}^2I_d)$ such that $\nabla_\tau f(\tau; \cdot)+Z$
% satisfies $(\epsilon,\delta)$-LDP.
Our network DP guarantee is stated below, again in a simplified asymptotic
form.
\begin{thm}
\label{thm:iteration}
Let $\eps<1$, $\delta < 1/2$. Alg.~\ref{algo:sgd} with
$\eta\leq2/\beta$ achieves $
(\eps',\delta+ \hat{\delta})$-network
DP for all $\hat{\delta}>0$ with
\begin{equation}
\label{eq:completegd_bigO}
\eps' = \sqrt{2q \ln (1/\delta)}\eps/\sqrt{\ln(1.25/\delta)},
\end{equation}
where $q = \max \big(\frac{2 N_u \ln n}
{n}, 2 \ln(1/\delta) \big)$ and $N_u = \frac{T}{n} + \sqrt{
\frac{3}{n} T \ln (1/\hat{\delta})}$.
\end{thm}

\begin{proof}[Sketch of proof]
The proof tracks the evolution of the privacy loss using Rényi
Differential Privacy (RDP) \citep{RDP} and leverages amplification by iteration
\citep{amp_iter} in a novel decentralized context. We
give here a brief sketch 
(see Appendix~\ref{sgdproof} for details).
% (see \ref{sgdproof} for details).
Let us fix two users $u$ and $v$ and bound the privacy leakage of $u$ from the
point of view of $v$. 
We again bound
the number of contributions $N_u$ of user $u$,
% as in
% Theorem~\ref{thm:complete}.
but unlike in the proof of
Theorem~\ref{thm:complete} we
apply this result to the user releasing information (namely $u$).
We then
compute the network RDP guarantee for a fixed
contribution of $u$ at time $t$. Crucially, it is sufficient to consider the
first time
$v$ receives the token at a step
$t'>t$. Privacy amplification by iteration tells us that the larger
$t'$,
the less is learned by $v$ about the
contribution of $u$.
Note that $t'$ follows a
geometric law of parameter $1/n$.
Using the weak convexity of the Rényi divergence \citep{amp_iter}, we can bound the Rényi divergence $D_
{\alpha}(Y_v|| Y'_v)$
between two random executions $Y_v$ and $Y'_v$ stopping at $v$ and differing
only in the contribution of $u$ by the expected divergence over the geometric
distribution. Combining with amplification by iteration eventually gives us
$D_{\alpha}(Y_v|| Y'_v) \leq 
4 \alpha L^2 \ln n / \sigma^2 n$.
We apply the composition property of RDP over the $N_u$
contributions of $u$ and convert
the RDP guarantee into $(\eps,
\delta)$-DP.
%Finally, as for Theorem~\ref{thm:complete}, we bound separately the
%privacy loss due to ``spotted'' contributions and conclude by adding the two terms.
\end{proof}

Algorithm~\ref{algo:sgd} is also a natural approach to
private SGD in the local model, and achieves $\eps'=O(\sqrt{N_u\ln
(1/\delta')}\eps)$ under LDP. Thus, for
$T=\Omega(n^2\sqrt{\ln(1/\delta)}/\ln n)$ iterations,
Theorem~\ref{thm:iteration} gives a privacy amplification of $O(\ln n/
\sqrt{n})$ compared to LDP. Measuring utility as the
amount of noise added to the gradients, the privacy-utility trade-off of
Algorithm~\ref{algo:sgd} in network DP is thus
nearly the same (up to a log factor) as that of private SGD in the trusted
curator model!\footnote{Incidentally, the analysis of centralized private SGD \citep{Bassily2014a}
also sets the number of iterations to be of order $n^2$.} For smaller $T$, the
amplification is much stronger than suggested by the simple closed form in
Eq.~\ref{eq:completegd_bigO}:
we can numerically find a smaller $\eps'$ that satisfy the conditions required
by our non-asymptotic result, see
% supplementary for details.
Appendix~\ref{sgdproof} for details.

We note that we can easily obtain utility guarantees for
Algorithm~\ref{algo:sgd} in terms of optimization
error. Indeed, the token performs a random walk on a complete graph so
the algorithm performs the same steps as a \emph{centralized}
(noisy) SGD algorithm. We can for instance rely on a classic result by 
\citet[][Theorem 2 therein]{shamir_zhang} which shows that SGD-type
algorithms
applied to a convex function and bounded convex domain
converge in $O(1/\sqrt{T})$ as long as gradients are unbiased with bounded
variance.

\begin{prop}
\label{prop:sgd_utility}
Let the diameter of $\mathcal{W}$ be bounded by $D$. Let $G^2=
L^2+\frac{8dL^2\ln(1.25/\delta)}{\eps^2}$, and $\tau\in\mathcal{W}$ be the
output
of Algorithm~\ref{algo:sgd} with $\eta=D/G\sqrt{t}$. Then we have:
$$\mathbb{E}[F(\tau) - F(w^*)] \leq 2DG(2+\log T)/\sqrt{T}.$$
\end{prop}

A consequence of Proposition~\ref{prop:sgd_utility} and
Theorem~\ref{thm:iteration} is that for
fixed privacy budget and
sufficiently large $T$, the expected error of
Algorithm~\ref{algo:sgd} is $O(\ln n / \sqrt{n})$ smaller
under network DP than under LDP.

\subsection{Discussion}
\label{sec:complete-discussion}

An advantage of considering a random walk over a complete graph is that our
approach is naturally
robust to the
presence of a (constant) number of colluding users. Indeed, when $c$ users
collude,
they can be seen as a unique node in the graph with a transition probability
of $\frac{c}{n}$ instead of $\frac{1}{n}$. We can then easily adapt the proofs
above, as the total number of visits to colluding
users follows $\mathcal{B}(T, c/n)$ and the size of a cycle between two
colluding users follows a geometric law of parameter $1-c/n$. Hence, we
obtain the same guarantees
under Definition~\ref{def:network_dp_col} as for the case with $n/c$
non-colluding users under Definition~\ref{def:network_dp}.
Interestingly, these privacy guarantees hold even if
colluding users bias their choice of the next user instead of choosing it
uniformly. Indeed, as soon
as colluded nodes do
not hold the token, the random walk remains unbiased, with the
same distribution for the time it takes to return to colluders (i.e., $
\mathcal{G}(m/n)$ for $m$ colluded nodes).

We note that despite the use of a complete graph, all users do not necessarily
need to be available throughout the process. For instance, if we assume that the availability of a user at each time step follows the same
Bernoulli distribution for every user, we can still build a random walk with
the
desired distribution, similarly to what is done in another context by 
\cite{random_check_ins}.

On the other hand, the assumption that users do not know the identity of the
previous sender and
the next receiver may seem quite strong.
It is however possible to lift this assumption by bounding the number of times
that a contribution of a given user $u$ is directly observed by a given user
$v$
separately and adding the corresponding privacy loss to our previous results.
This additional term dominates the others for small values of $T$ due to its
large variance (an ``unlucky'' node may forward a lot of times the token to
the same node). But as the expected number of contributions per node increases,
the relative importance of this term decreases (and thus the privacy
amplification increases) until $T = \Omega
(n^2)$, for which the amplification reaches the same order as in
Theorems~\ref{thm:complete}-\ref{thm:iteration}. Even though $T = \Omega
(n^2)$ is seldom used in practice, we note that
it is
also required to obtain optimal privacy amplification by iteration under
multiple contributions per user \citep{amp_iter}. Moreover, in practical
implementations, we can mitigate the large variance effect in the regime where
$T = o(n^2)$ by enforcing a
\emph{deterministic} bound on the
number of times any edge $(u,v)$
is used, e.g., by contributing only noise along $(u,v)$ after it has been used
too many times. We refer to Appendix~\ref{spotted} for details and formal
derivations.

%% file: expes.tex
% !TEX root = main_supp.tex

%\begin{figure*}[t] 
%    \centering
%  \begin{subfigure}[b]{0.41\linewidth}
%    \centering
%    \includegraphics[width=.89\linewidth]{fig/theoric_logscale_eps01.pdf}
%    \caption{Real summation (theoretical bounds)\label{fig:fixed_contrib}}
%  \end{subfigure}
%  \begin{subfigure}[b]{0.41\linewidth}
%    \centering
%    \includegraphics[width=.89\linewidth]{fig/gaussian.pdf}
%    \caption{Real summation (empirical)\label{fig:gaussian}}
%  \end{subfigure}
%  \caption{Comparing network and local DP on the task of real summation for
%  $T=100n$. $\eps_0$ rules the amount of local noise added to each
%  contribution (i.e., each
%  single contribution taken in isolation satisfies $\eps_0$-LDP). For the
%  empirical results of Figure~\ref{fig:gaussian}, the curves report the
%  average privacy loss across all pairs of users and all $10$ random runs,
%  while their error bars give the best and worst cases.}
%\end{figure*}

\section{EXPERIMENTS}
\label{sec:exp}

We now present some numerical experiments that
illustrate the practical significance of our privacy amplification
results in the complete graph setting (Section~\ref{sec:complete}).\footnote{The code is available \url{https://github.com/totilas/privacy-amplification-by-decentralization}}

\subsection{Real Summation}
\label{sec:exp_sum}

\textbf{Comparison of analytical bounds.} We numerically evaluate
the theoretical (non-asymptotic) bound of Theorem~\ref{thm:complete} for the
task of real summation and
compare it to local DP. Recall that the number
of contributions of a user is random (with expected value $T/n$).
For a fair comparison between network and local DP, we
derive an analogue of Theorem~\ref{thm:complete} for local DP.
% where we use the
% same Chernoff bound to control the number contributions of a
% user as well as advanced composition to measure the total privacy loss.
In addition, to isolate the effect of the number of contributions (which is
the same in both settings), we also report the bounds obtained
under the assumption that
each user contributes exactly $T/n$ times.
Figure~\ref{fig:fixed_contrib} plots the
value of the bounds for varying $n$. We see that our theoretical result
improves upon local DP as
soon as $n\geq 20$, and these gains become more significant as $n$
increases. We note that the curves obtained under the fixed number of
contributions per user also suggest that a better control of $N_v$ in the
analytical bound could make our amplification result significantly tighter.

\begin{figure}[t]
  \centering
  \begin{subfigure}[b]{0.48\linewidth}
    \centering
    \includegraphics[width=\linewidth]{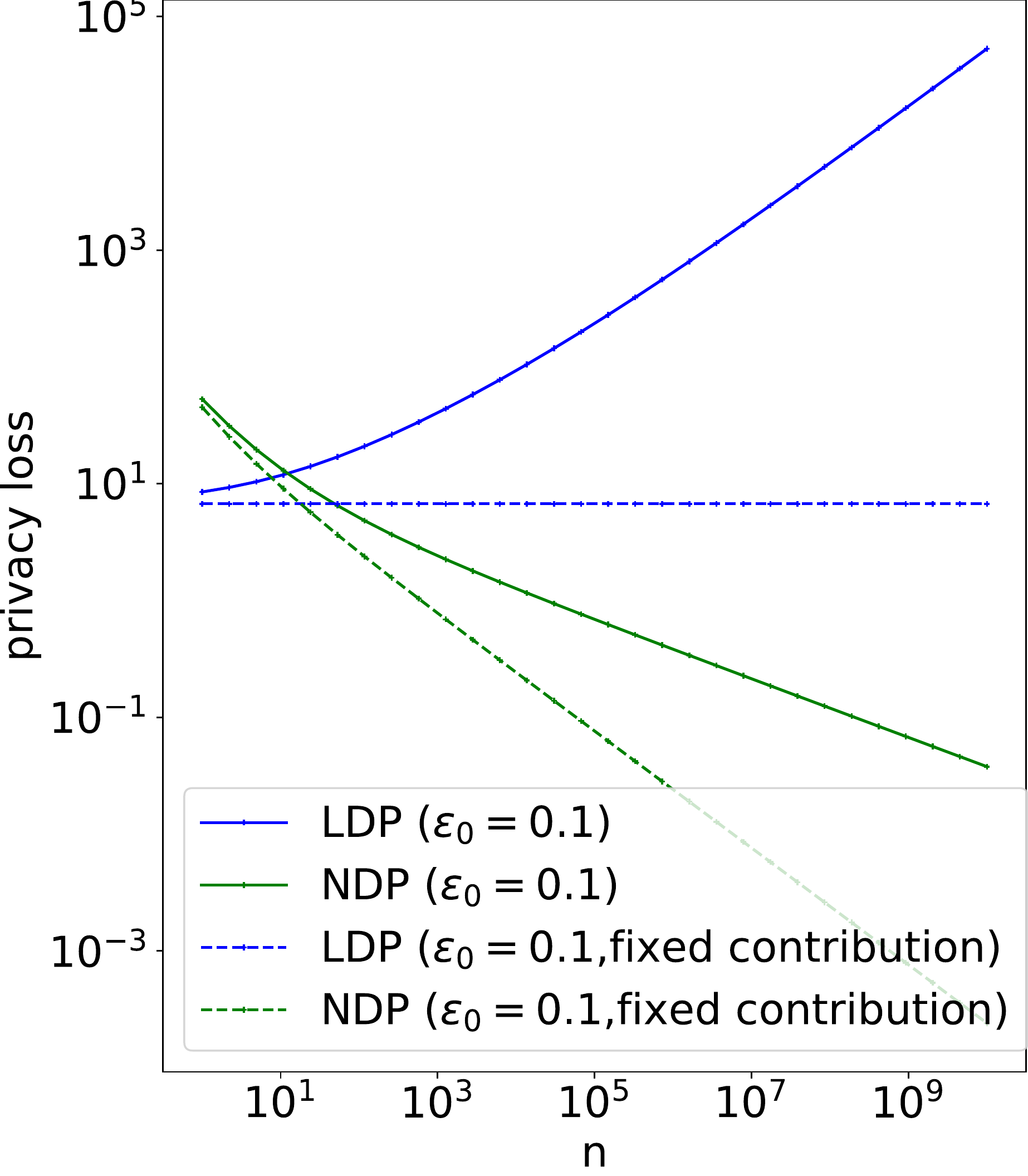}
    \caption{Theoretical bounds\label{fig:fixed_contrib}}
  \end{subfigure}
  \begin{subfigure}[b]{0.48\linewidth}
    \centering
    \includegraphics[width=\linewidth]{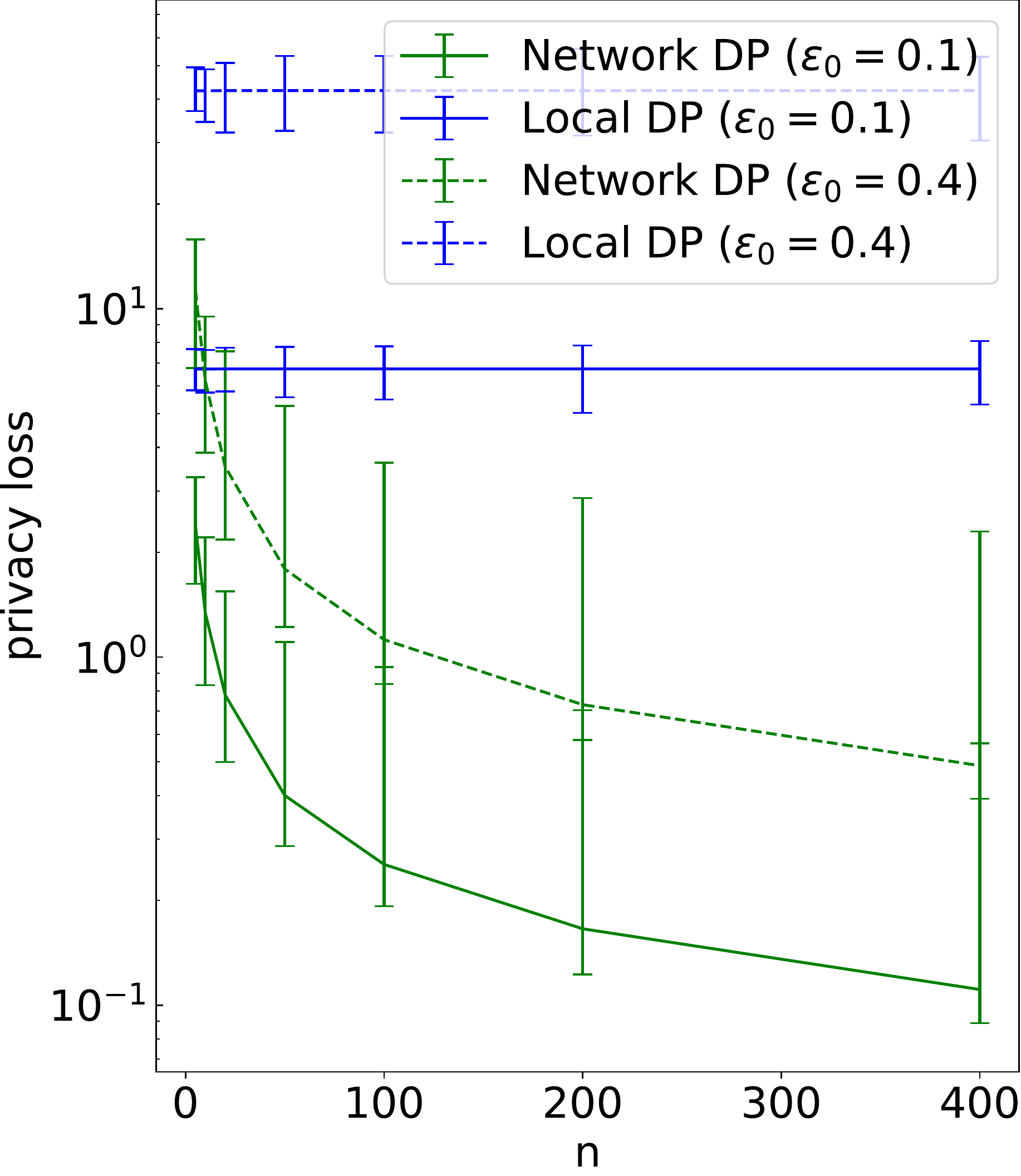}
    \caption{Empirical\label{fig:gaussian}}
  \end{subfigure}
  \caption{Comparing network and local DP on real summation for
  $T=100n$. $\eps_0$ rules the amount of local noise added to each
  contribution (i.e., each
  single contribution taken in isolation satisfies $\eps_0$-LDP). For the
  empirical results of Figure~\ref{fig:gaussian}, the curves report the
  average privacy loss across all pairs of users and all $10$ random runs;
  error bars give best and worst cases.}
\end{figure}

\textbf{Gap with empirical behavior.}
Our formal analysis involves controlling the
number of contributions of users, as well as the size of
cycles using concentration inequalities, which require
some approximations. In practical deployments one can instead use
the actual values of these quantities to compute the privacy loss. We thus
investigate the gap between our theoretical guarantees and what can be
obtained in practice through simulations.
Specifically, we sample a
random walk of size $T=100n$. Then, for each pair of users, we compute the
privacy loss based on the actual walk and the
advanced composition mechanism.
We repeat this experiment over 10 random walks and we can then report the
average, the best
and the worst privacy loss observed across all pairs of users and all random
runs.
Figure~\ref{fig:gaussian} reports such empirical results obtained for the case
of real summation with the Gaussian mechanism, where the privacy grows
with a factor $\sqrt{m}$ where $m$ is the number of elements aggregated
together (i.e., the setting covered by
Theorem~\ref{thm:complete}). We observe that the
gains achieved by network DP are significantly stronger in practice than what
our theoretical bound guarantees, and are significant even for small $n$ (see
Figure~\ref{fig:fixed_contrib}).
 % These results again
% suggest that there is room for improvement in our analysis, for instance
% by resorting to better concentration bounds.
Our experiments on discrete
histogram computation also show significant gains (see Appendix~\ref{exp_supp}).
% that the
% empirical gains
% from privacy amplification by decentralization are also
% significant for this task.
% \ref{exp_supp}.

\subsection{Machine Learning with SGD}
\label{sec:exp_sgd}

\begin{figure*}[t]
  \centering
  \begin{subfigure}[b]{0.221\linewidth}
    \centering
    \includegraphics[width=\linewidth]{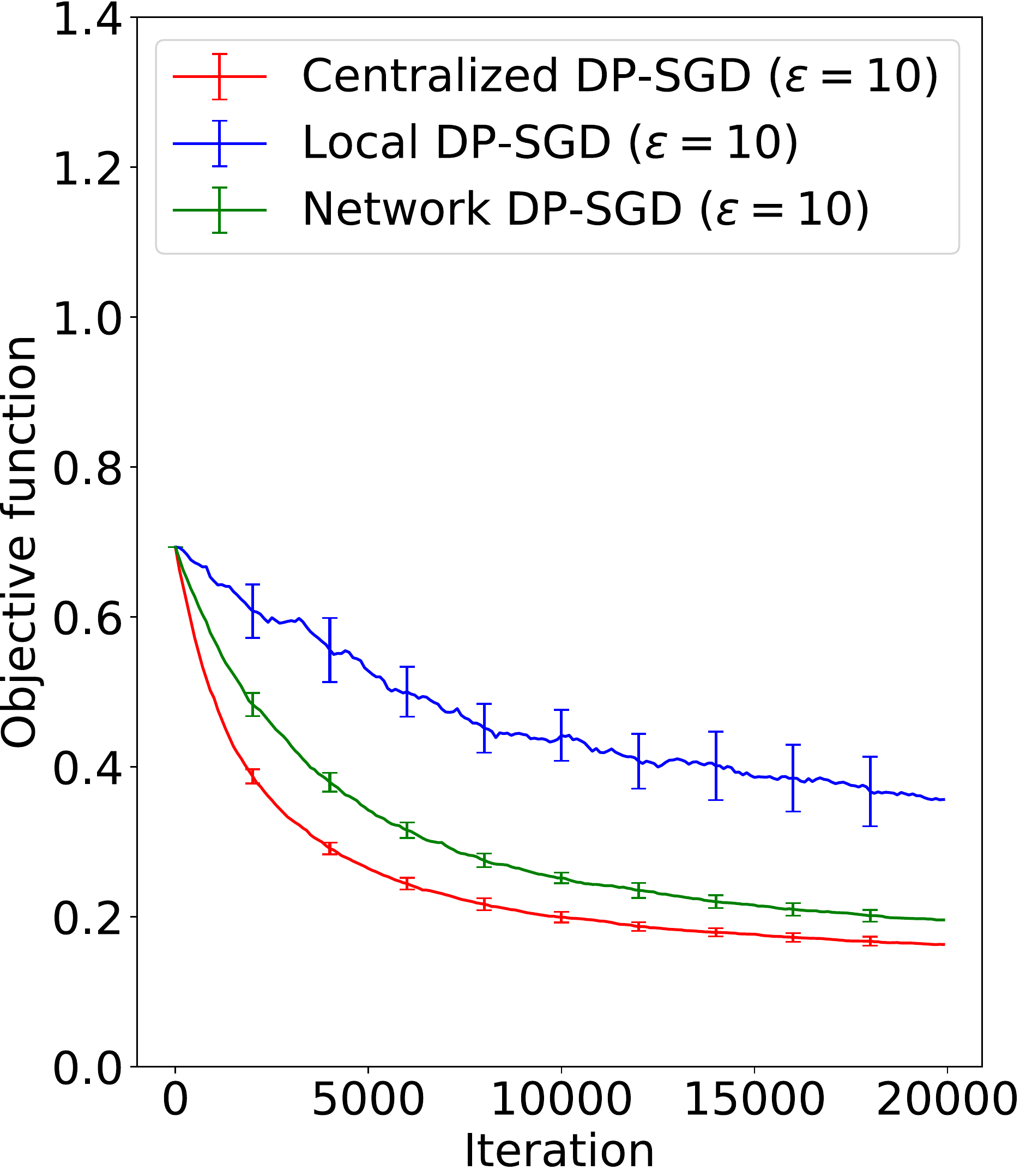}
    %\caption{Training objective on the dataset Houses ($\eps = 10$) \label{fig:objfun}}
  \end{subfigure}
  \begin{subfigure}[b]{0.221\linewidth}
    \centering
    \includegraphics[width=\linewidth]{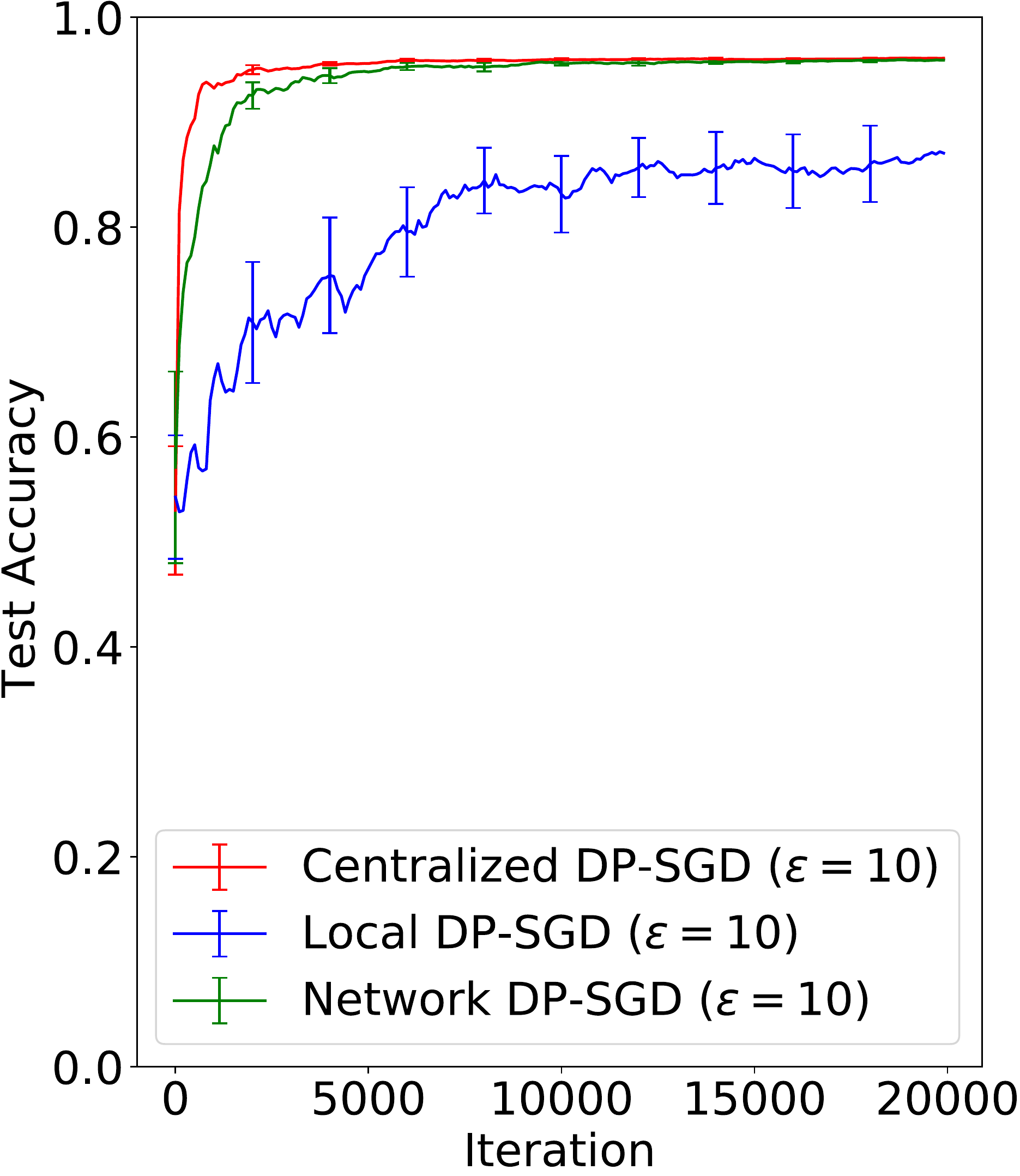}
    %\caption{Accuracy on the dataset Houses ($\eps = 10$)\label{fig:accuracy}}
  \end{subfigure}
  \begin{subfigure}[b]{0.221\linewidth}
    \centering
    \includegraphics[width=\linewidth]{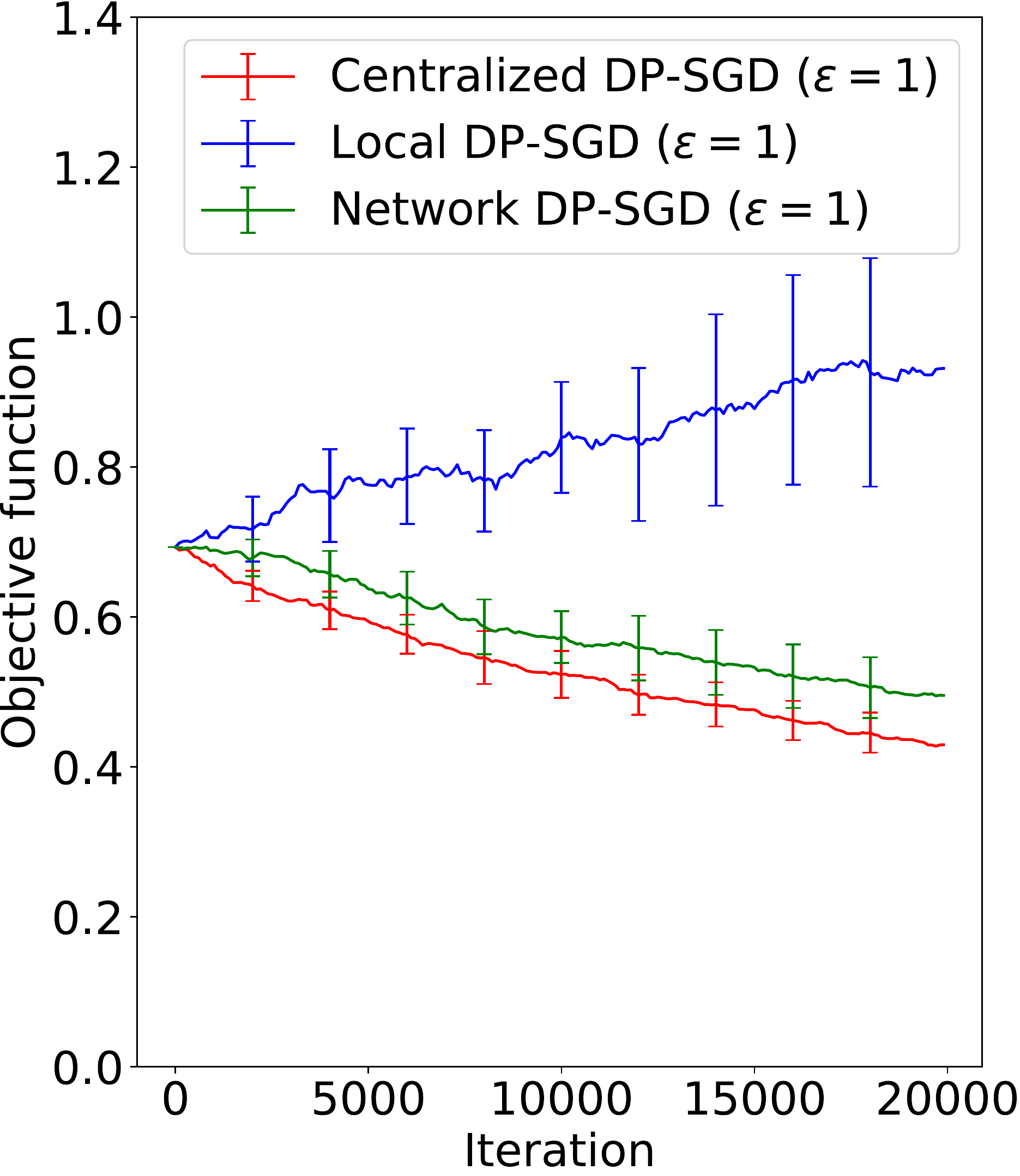}
    %\caption{Training objective on the dataset Houses ($\eps = 1$)\label{fig:objfun_eps}}
  \end{subfigure}
  \begin{subfigure}[b]{0.221\linewidth}
    \centering
    \includegraphics[width=\linewidth]{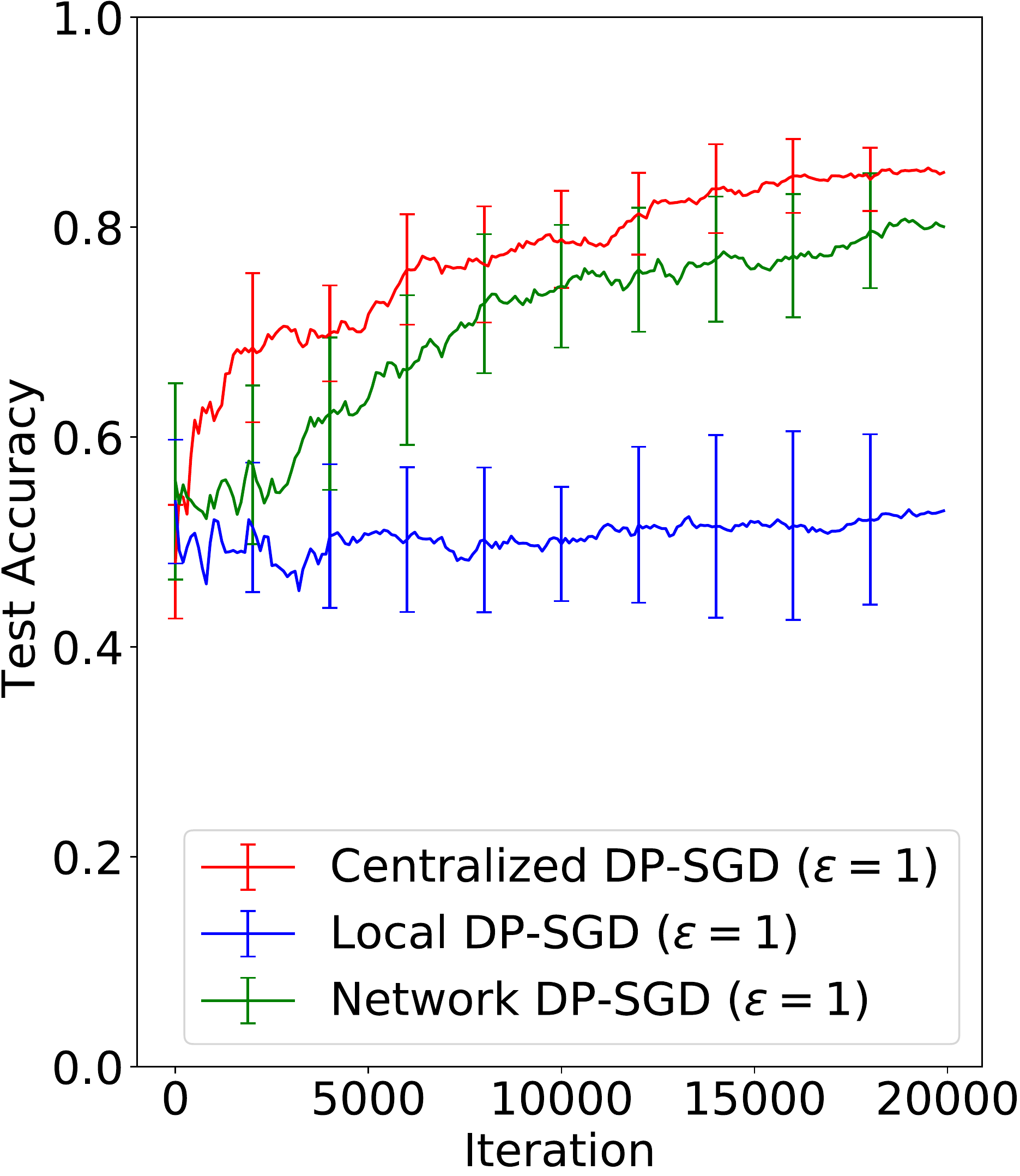}
    %\caption{Accuracy on the dataset Houses ($\eps = 1$)\label{fig:accuracy_eps}}
  \end{subfigure}
  \caption{Comparing three settings for SGD with
  gradient perturbation. Unlike Local and Network DP-SGD, Centralized
  DP-SGD requires a trusted curator
  and benefits from amplification by subsampling.
  % % In contrast,
  % local DP-SGD requires a large amount of noise.
  Network DP nearly bridges the
  gap between Centralized and Local DP-SGD. In all methods, $\sigma$
  is set to ensure $\eps=10$ (left plots) or $\eps=1$ (right plots), and
  $\delta=10^{-6}$. Mean and standard deviations are computed over $20$
  runs.}
  \label{fig:sgd}
\end{figure*}

We now present some experiments on the task of training a logistic regression
model in the decentralized setting. Logistic regression corresponds to solving
Eq.~\ref{eq:optim} with $\mathcal{W}=\mathbb{R}^d$ and the loss functions
defined as $f(w;D_u)=\frac{1}{|D_u|}\sum_{(x,y)\in D_u}\ln(1+\exp(-yw^\top
x))$ where $x\in
\mathbb{R}^d$ and $y\in \{-1,1\}$.
We use a binarized version of UCI Housing dataset.\footnote{\url{https://www.openml.org/d/823}}
% \footnote{Download link:
% \url{https://www.openml.org/d/823}}
We standardize the features and
further normalize each
data point $x$ to have unit L2 norm so that the logistic loss is
$1$-Lipschitz for any $(x,y)$.
We split the dataset uniformly at random into a training set (80\%) and a
test set, and further split the training set across $n=2000$ users, resulting
in each user $u$ having a local dataset $D_u$ of size $8$.

We compare three variants of private SGD based on gradient
perturbation with the Gaussian mechanism. \emph{Centralized DP-SGD} is the centralized version of
differentially private SGD introduced by 
\cite{Bassily2014a}, which assumes the presence of a trusted
curator/aggregator.
\emph{Local DP-SGD} corresponds to Algorithm~\ref{algo:sgd}
with the noise calibrated for the LDP setting.
Finally, \emph{Network DP-SGD} is Algorithm~\ref{algo:sgd} 
with the noise calibrated according to network DP (see
Theorem~\ref{thm:iteration}). To make the comparison as fair as
possible, all approaches (including Centralized DP-SGD) use the full dataset
$D_u$ of a randomly chosen user $u$ as the mini-batch at each step.

Given the privacy budget $(\eps,\delta)$ for the whole procedure, each of the three
methods leads to a different choice for $\sigma$ that parametrizes the level
of noise added to each gradient.
In our experiments, we fix $\eps = 10$ (low privacy) and $\eps = 1$ 
(stronger privacy) and $\delta = 10^{-6}$. We
recall that we consider user-level DP ($X \sim_u X'$
differ in the local database of user $u$).
Note that due to composition, more iterations increase the per-iteration
level of noise needed to achieve a fixed DP guarantee.
As the number of contributions of a
given user is random, we upper bound it in advance with a tighter
bound than used in
our theorems, namely $c T/n$ where $c$ is a parameter to tune. If a user is
asked to participate more times than budgeted, it
simply forwards the token to another user without adding any contribution. In
the case of Network DP-SGD, the user still adds noise
as the privacy guarantees of others rely on it.
Note that the best regime for network DP is when the number of contributions
of a user is roughly equal to $n$,
see Theorem~\ref{thm:iteration}. 
In our experiments, we are not in this regime but the
privacy amplification effect is stronger than the closed form of the theorem. In
practice, we compute
numerically the smallest $\sigma$ needed to fulfill the conditions
of the proof (see % supplementary).
Appendix~\ref{sgdproof}).

Figure~\ref{fig:sgd} shows results for $T=20000$, where the step size $\eta$
was tuned separately for each approach in $[10^{-4},2]$.
We see that
Network DP-SGD
nearly matches the privacy-utility trade-off of Centralized DP-SGD for both
$\eps=1$ and $\eps=10$ without relying on a trusted curator.
Network DP-SGD
also clearly outperforms Local DP-SGD, which actually diverges for $\eps=1$.
These empirical results are consistent with our theory and show that Network
DP-SGD significantly amplifies privacy compared to
local DP-SGD even when the
number of iterations $T$ is much smaller than $O(n^2/\ln n)$, a regime which
is of much practical importance.

%% file: conclu.tex
% !TEX root = main_supp.tex

\section{PERSPECTIVES}
\label{sec:conclu}

We believe that our work opens many interesting perspectives.
% Aside from making our bounds for the complete graph setting
% more tight,
We would like to consider generalizations of our results to arbitrary
graphs by relying on classic graph
theoretic notions like the hitting time. Furthermore, we think
that time-evolving topologies can help improve robustness to
collusions, in particular in rings and other sparse topologies.
Network DP can also be used to study other
decentralized models of computation.
% with the goal of proving privacy amplification results.
A
natural extension of the algorithms studied here is
to consider multiple tokens walking on the graph in parallel. We would also
like to study randomized gossip algorithms \citep{random_gossip}, which are
popular for decentralized optimization in machine learning \citep{Colin2016a}
and were recently shown to
provide DP guarantees in the context of rumor spreading 
\citep{Bellet2020a}.
Finally, we would like to investigate the fundamental limits of network DP and
consider further relaxations where users put more trust in their direct
neighbors than in more distant users.

%% file: ack.tex
% !TEX root = main_supp.tex

\subsubsection*{Acknowledgements}
We thank the reviewers for their constructive feedback that helped us to
significantly improve the proof for real aggregation in the complete
graph.

This work was supported by grants ANR-16-CE23-0016 (Project PAMELA) and
ANR-20-CE23-0015 (Project PRIDE). The PhD scholarship of Edwige Cyffers is
funded in part by Région Hauts-de-France.

%% file: appendix.tex
% !TEX root = main_supp.tex

% Supplementary material: To improve readability, you must use a single-column format for the supplementary material.
\newpage
\onecolumn
\appendix
% \aistatstitle{Supplementary Material\\Privacy Amplification by
% Decentralization}

\section*{APPENDIX}

\section{Proof of Theorem~\ref{thm:ring-sum} (Real Aggregation on a Ring)}
\label{ringproof}

\begin{proof}
We start by proving the utility claim. Algorithm~\ref{algo:real_ring} adds
independent noise with standard deviation $\sigma_{loc}$ to the token every $n-1$
contributions. As there are $Kn$ steps, such noise is added $\lfloor K n/
(n-1)\rfloor$ times. By commutativity, the total noise has standard deviation $\sqrt{\lfloor K n/(n-1)\rfloor} \sigma_{loc}$.

We now turn to the network differential privacy claim. 
Let us fix two distinct users $u$ and $v$ and consider what $v$ learns about
the data of $u$. Recall that
the structure of the ring is
assumed to be public. The view $\Obs{v}$ of $v$ (i.e., the information observed by $v$
during the execution of the protocol as defined in Eq.~\ref{eq:obs_abstract})
thus corresponds to
the $K$
values of the token that she receives. We denote these values by $\tau_1^v,
\dots,
\tau_K^v$, each of them corresponding to user
contributions aggregated along
with random noise. We define the view of $v$ accordingly as:
\begin{equation}
\label{eq:obs}
\Obs{v}(\mathcal{A}(D)) = ( \tau_i^v )_{i=1}^K.
\end{equation}

Let us fix $i \in \{2,\dots,K\}$. By construction, $\tau_{i}^v-\tau_{i-1}^v$
is equal to
the sum of updates
between two visits of the token. In particular, we have the guarantee that at least one
user different from $v$ has added noise in $\tau_{i}^v-\tau_{i-1}^v$ (as there
are $n>n-1$ steps), and that $\tau_{i}^v-\tau_{i-1}^v$ does not contain more
than one
contribution made by $v$. It follows that the aggregation $\tau_{i+1}^v-\tau_
{i}^v$ can be rewritten as $\text{Perturb}(x_u^i;\sigma_{loc}) + z$, where $z$ is
independent from the contribution of $u$. By the $(\eps,\delta)$-LDP property
of $\text{Perturb}
(\cdot;\sigma_{loc})$ and the post-processing
property of differential privacy, we have for any $x,x'$:
\[ \Pro(\tau_{i+1}^v-\tau_{i}^v = \tau | x_u^i=x) \leq e^{\eps} \Pro(\tau_
{i+1}^v-\tau_{i}^v = \tau | x_u^i=x') + \delta. \]

For the first token $\tau_1^v$, note it also contains noise with standard
deviation $\sigma_{loc}$ added by the first user, so the same
guarantee holds.

Finally, we apply the advanced
composition theorem \citep{boosting} to
get a differential privacy guarantee for the $K$ visits of the
token, leading to the final privacy guarantee of $(\sqrt{2 K \ln 
(1/\delta')}\eps+ K \eps(e^{\eps}-1),K  \delta +\delta')$-network DP.
\end{proof}

\section{Proof of Theorem~\ref{thm:ring-histogram} (Histogram Computation on a
Ring)}
\label{histoproof}
\begin{proof}

The proof is similar in spirit to the real summation case (see
Appendix~\ref{ringproof}), but leverages privacy amplification by subsampling
to be able to quantify how much information is leaked by
the value of the token (which is now a histogram).

We start by the utility claim (expected number of contributions). There
are $Kn$ steps with at each step a probability $\gamma$ of adding a random
response, plus the $\gamma n$ random responses at initialization,
leading to a total of  $\gamma n(K+1)$ random responses in expectation.

We now turn to the differential privacy guarantee. 
The view of a user $v$ is the content of the token at each visit
of the token as defined in Eq.~\ref{eq:obs}, except that each $\tau_i^v\in
\mathbb{N}^L$ is now a histogram over the domain $[L]$. More specifically,
for $i \in \{2,\dots,K\}$, the
difference $\tau_{i+1}^v-\tau_{i}^v$ between two consecutive
tokens is
now a discrete histogram of $n$ answers obtained by $RR_{\gamma}$ (each
of them is random with probability $\gamma$). Similarly, in the first round, 
the token is initialized with $\gamma n$ random elements.
Therefore, we can apply
results from amplification by shuffling, because a discrete histogram carries
the same more information as a shuffle of the individual values. In
particular, we can use Corollary~9 from
\cite{amp_shuffling} that we recall below.

\begin{thm}[Erlingsson]
  Let $n\geq 100, 0 < \eps_0 < \frac{1}{2}$ and $\delta < \frac{1}{100}$. For
  a local randomizer ensuring $\eps_0$-LDP, the shuffling
  mechanism is $(\eps, \delta)$-differentially private with
  \[ \eps = 12 \eps_0 \sqrt{\frac{\log(1/\delta)}{n}}. \]
\end{thm}

We can apply this result to the information revealed by the value of the token
between two visits to user $v$. The required LDP guarantee is ensured by the
use of the randomized response mechanism, where we set $\gamma$
so that $RR_{\gamma}$ satisfies $12\eps \sqrt{
\frac{\log(1/\delta)}{n}}$-LDP, leading to an $(\eps,\delta)$-DP guarantee
after shuffling.
We conclude by the application of advanced composition \citep{boosting}.
\end{proof}

\section{Relation between Network DP on a Ring and Pan-Privacy}
\label{app:pan_privacy}

In this section, we highlight an interesting connection between the specific
case of network DP on a ring topology and the pan-privacy model 
\citep{pan-privacy}.
In the pan-privacy model, raw data is processed in an online fashion by a
central party. This central party is trusted to process raw data but not to
store it in perpetuity, and its storage may be subject to breaches (i.e., its
internal state may become visible to an adversary). It can thus be seen
as an intermediate trust model between the central and local models.
Connections between pan-privacy and the shuffle model have been recently
studied \citep{pan-shuffling}, allowing in some cases to adapt algorithms from
one setting to the other. Other recent work has studied the relation
between pan-privacy with several breaches and the local model 
\citep{pan-local}.

% is done
% sequentially on a data stream by a trusted curatorbe done on the raw data as
% in the central model, but with more restrictive assumptions, requiring stronger protection such as the local setting. More precisely, it has been connected to the shuffling model \cite{pan-shuffling}, allowing to adapt in some cases algorithms to one setting to another. Link has been showed between allowing several intrusions in the pan private model and the local model \cite{pan-local}.

To formally define pan-privacy, we first need to define what we mean by
online algorithms. An
online algorithm receives a stream of raw data and sequentially updates an
internal state with one data point before deleting it. At the end of the
stream, the algorithm publishes a final output based on its last
internal state.

\begin{defn}[Online algorithm] An online algorithm $\mathcal{A}$ is defined by
a sequence of internal algorithms $\A_1,\dots$ and an output algorithm
$\A_O$. Given an input stream $\vec{x}$, the first internal algorithm
$\A_1:\mathcal{X}\rightarrow \mathcal{I}$ maps $x_1$ to a state $s_1$, and
for $i\geq
2$, $\A_i:\mathcal{X}\times \mathcal{I}\rightarrow \mathcal{I}$ maps
$x_i$ and the previous state $s_{i-1}$ to a new state $s_i$. At the end of the
stream, $\A$ publishes a final output by executing $\A_O:
\mathcal{I}\rightarrow O$ on its final internal state. We denote by $\A_
{\mathcal{I}}(\vec{x})$ the internal state of $\A$ after processing stream
$\vec{x}$.
\end{defn}

In pan-privacy, the algorithm is
trusted to process a raw data stream, but should protect its
internal states against potential breaches. The
moment of the update where the state is modified by a raw data point is
supposed to be atomic. Hence, the observable impact of a data point is
restricted to the internal state
and the final output. Below, we state the standard definition of pan-privacy
with a single breach, i.e., the adversary may observe a single internal state
in addition to the final output. Two streams $\vec{x},\vec{x}'$ are said to be
neighboring if they differ in at most one element.

\begin{defn}[Pan-privacy] An online algorithm $\mathcal{A}$ is $(\eps,
\delta)$-pan private if for every pair of neighboring streams $\vec{x} \sim
\vec{x}'$, for every time $t$, and for every subset $T\subseteq
\mathcal{I}\times \mathcal{O}$, we have:
\[
\mathbb{P}\left(\left(\mathcal{A}_{\mathcal{I}}\left(\vec{x}_{\leq t}\right), 
\mathcal{A}_{O}\left(\mathcal{A}_{\mathcal{I}}(\vec{x})\right)\right) \in T\right) \leq
e^{\eps} \cdot \mathbb{P}\left(\left(\mathcal{A}_{\mathcal{I}}\left(\vec{x}_{\leq
t}'\right), \mathcal{A}_{O}\left(\mathcal{A}_{\mathcal{I}}\left
(\vec{x}'\right)\right)\right) \in T\right)+\delta,
\]
where $\vec{x}_{\leq t}$ denotes the first $t$ elements of stream $\vec{x}$.
\end{defn}

We can now make a connection between the above pan-privacy definition and our
simple protocols for network DP on a ring topology introduced in
Section~\ref{sec:ring}, in the case where each user
contributes only once ($K=1$ in our notations). In
our network DP setting, the internal state corresponds to the value of
the token and the final output is empty (or is equal to the final state of
the token, if one performs an additional cycle over the ring during which the
token is left unchanged to broadcast it to all users). A breach at time $t$ 
(i.e., observation of internal state $s_t$) corresponds to the observation of
the token by the $t$-th user. Note also that our neighboring relation on the
users' datasets is equivalent to that on data streams for the case of $K=1$. Therefore, we can
simulate a pan-private algorithm as a network DP algorithm on a ring.
% Indeed, the neighboring relation is the same, and the condition on time $t$ corresponds to the constraint for the node participating at time $t$ in the network differential privacy.
% The sharing of the final state is not included in the presented algorithm but can easily be included as a additional round where the token is left identical and only broadcast from one node to another.

We note that the lack of privacy gains for network DP compared to local DP
when considering a ring topology with collusions (see discussion in
Section~\ref{sec:ring-discussion}) is in line with the reduction of
\cite{pan-local}, which shows that in pure pan-privacy, protection
against multiple breaches is equivalent to sequentially interactive local
privacy.

While network DP reduces to pan-privacy
when the topology is the ring and one considers simple protocols with a single
token and a single contribution per user, we emphasize that our model is
more general and
potentially allows superior privacy-utility trade-offs for more complex
protocols and/or topologies.
This is illustrated by our results on the complete graph, where breaches
cannot follow an arbitrary pattern. Indeed, as a breach corresponds
to sending the token to a colluding user, this risk is mitigated by the
properties of the random walk: as long as the token is held by non-colluding
users, the walk stays unbiased and thus does not return too quickly to
colluding users. This additional structure on the potential breaches give us
the room for stronger guarantees.

\section{Proof of Theorem~\ref{thm:complete} (Real Summation on the Complete
Graph)} \label{completeproof}

\begin{proof}
We will prove an $(\eps_f, \delta_f)$-DP guarantee for Algorithm~\ref{algo:complete}. We note that our proof does not require
the global time counter $t$ to be hidden from users (i.e., the result holds
even if users receiving the token know how many users have added
contributions to the token since its last visit).

Let us fix two distinct users $u$ and $v$. We aim to quantify how much
information about the private data of user $u$ is leaked to $v$ from
the visits of the token. 
Recall that we assume the token path to
be hidden, including the previous sender and the next
receiver. We can thus define the view $\Obs{v}$ of user $v$ by:
\begin{equation}
\label{eq:obs_complete}
\Obs{v}(\mathcal{A}(D)) =  (\tau_{k_i})_{i=1}^{T_v},
\end{equation}
where $k_i$ is the $i$-th time that $v$ receives the token, $\tau_{k_i}$ the
corresponding value of the token, and $T_v$ the
number of times that $v$ had the token during the whole
execution of the protocol.

We aim at bounding the privacy loss with respect to the
contributions of $u$ from the point of view of $v$. We call ``cycle'' the
portion of the walk between two visits of the token to $v$. We first note
that we can decompose the walk in cycles by cutting the walk at each $k_i$. If
a contribution of $u$ happens at time $t$, there is single $i$ such a $k_i < t < k_{i+1}$.\footnote{If the contribution of $u$ occurs before
the first passage of the token at $v$, we can take $k_i=0$. As for
contributions
occurring \emph{after} the last passage of the token at $v$, they do not
incur any privacy loss.}
Note that the token values observed before $t$ do not depend on the
contribution of $u$ at time $t$. Moreover, it is sufficient to bound the
privacy
loss induced by the observation of the token at $k_{i+1}$: indeed, by the
post-processing property of DP, no additional privacy loss with respect to
$v$ will occur for observations posterior to $k_{i+1}$. 

To allow the use of privacy amplification by subsampling results 
\citep{Balle_subsampling}, we will actually consider a variant of the actual
walk. We assume that if $n$ steps have occurred since the last visit
of the token to $v$, the value of the token at that time is observed by $v$ ``for
free''. As the information
leaked to v by the actual walk can be obtained by
post-processing of this \emph{fictive} walk, it is sufficient to prove privacy
guarantees on the fictive walk.

The number of observations of the token by $v$ can be bounded by the ``real''
observations (from actual visits of the token) plus the fictive ones. By
definition,
there is no more than $T/n$ fictive observations of the token. We now bound
the number of real visits of the token to $v$.

As the user receiving the token at a given step is chosen uniformly at random
and independently from the other steps, there is a probability of $1/n$ that
the token is at $v$ at any given step. Thus, the number of visits $T_v$ to $v$
follows a binomial law $\mathcal{B}(T, 1/n)$. We bound it by $N_v$
with probability $1 - \hat{\delta}$ using Chernoff.
Recall that the Chernoff bound allows to upper bound (with high probability)
the sum of independent random variables $X_1,\dots,X_T$ of expected value $p$,
for any real $\alpha \in [0,1]$:
\[ \Pro\left(\sum_{i=1}^T X_i \geq (1+\alpha)Tp\right) \leq e^{-\alpha^2 pT /
3}. \]

In our case, we want to upper
bound the probability that the number of contributions $T_v$ of a given user
$v$ exceeds some threshold $N_v$ by $\hat{\delta}$. Using the previous bound
for $p = 1/n$ and $\alpha = \sqrt{\frac{3n \log(1/\hat{\delta})}{T}}$, by
considering the random variables equal to $1$ it $v$ has the token and $0$ otherwise, we have:
\[\Pro\bigg(T_v \geq \underbrace{\frac{T}{n} + \sqrt{\frac{3T}{n}\log(1/
\hat{\delta}) }}_{N_v}\bigg)\leq
\hat{\delta}. \]

Let us now upper bound the privacy loss that occurs during a fixed cycle. The
information revealed to $v$ by a
cycle of size $1 \leq m \leq n$ can be seen as a mechanism
$\mathcal{M}=\mathcal{A}\circ\mathcal{S}$, where $\mathcal{A}$ corresponds to
the aggregation of $m$ values with $m$ additions of Gaussian noise, and
$\mathcal{S}$ corresponds to subsampling with replacement $m$ users among $n$
(as
each user is uniformly chosen at random at each step).
The base mechanism $
\mathcal{A}$ satisfies $(\eps/ \sqrt{m}, \delta)$-DP.
%  The final mechanism is
% the aggregation of the sampled contribution where the $m$ local Gaussian noises sum up, leading to a mechanism $(\eps/ \sqrt{m}, \delta)$-differentially private.

According to  Theorem~10 from 
\cite{Balle_subsampling}: given $n$ users
and $m$ the size of the cycle, the privacy of $\mathcal{M}=\mathcal{A}\circ
\mathcal{S}$ satisfies $(\eps_{cycle},\delta_{cycle})$ with:
\begin{equation}
\label{eq:subscycle}
\eps_{cycle} = \log(1+ (1 - (1-1/n)^m)(e^
{\eps_{\A}}-1),
\end{equation}
and $\delta_{cycle}\leq \delta_{\mathcal{A}}$, where ($\eps_
{\A},\delta_{\A}$) is the
level of DP guaranteed by $\A$.
Hence, for a cycle of size $m$, $\mathcal{M}$ satisfies $(\eps_
{cycle},\delta)$-DP with \[\eps_{cycle} \leq \log \left(1 + (1 -
\left(1 - \frac{1}{n}\right)^m) (e^{\eps/\sqrt{m}} - 1)  \right).  \]

Using the fact that $\eps \leq 1$, we can upper bound $e^{\eps/\sqrt{m}} - 1$ by $2 \eps/\sqrt{m}$. Moreover, as $1/n < 0.58$, we have
$-\frac{3}{2n} \leq  \log(1 - 1/n)$. So we have
\[1 - \exp(m \log(1 - 1/n)) \leq 1 - \exp \left(-\frac{3m}{2n} \right) \leq \frac{3m}{2n}. \]
Combining the two upper bounds and the classical inequality $\log(1+x) \leq x$
gives us:

\[\eps_{cycle} \leq \frac{3\sqrt{m}\eps}{n} \leq \frac{3 \eps}{\sqrt{n}}. \]
Hence we can upper bound the privacy loss of each cycle by $\frac{3 \eps}{
\sqrt{n}}$ regardless of its length $m$. Finally, we use advanced composition
to account for the
privacy losses of all $T/n+N_v$ cycles, leading to the following bound:
\[ \eps_{f} \leq \sqrt{\left(\frac{4T}{n} + 2\sqrt{\frac{3T}{n}\log(1/
\hat{\delta})}\right)\ln(1/\delta')} \frac{3 \eps}{\sqrt{n}} + \sqrt{\frac{2T}{n} + \sqrt{\frac{3T}{n}\log(1/\hat{\delta})}}  \eps (e^{ 3\eps/\sqrt{n})} - 1),\]
with $\delta_f = (N_v + T/n) \delta + \delta' + \hat{\delta}$.
\end{proof}

\section{Histogram Computation on the Complete Graph}
\label{hist_complete}

For discrete histogram computation on the complete graph, we propose
Algorithm~\ref{algo:completeshuff}: when receiving the token, each user
perturbs his/her contribution with $L$-ary randomized response, adds it to the
token and forwards the token to another user chosen uniformly at random.
We have the following guarantees, which provide a privacy amplification of $O
(1/\sqrt{n})$ over LDP for $T=\Omega(n)$.

\begin{algorithm}[t]
  \centering
  \caption{Private histogram computation on a complete graph.}
  \label{algo:completeshuff}
  \begin{algorithmic}   
  \STATE Init. $\tau\in\mathbb{N}^L$    
  \FOR{$t = 1$ to $T$}
    \STATE Draw $u \sim \mathcal{U}({1,\dots, n})$\;
    \STATE $y_u^k \leftarrow RR_{\gamma}(x_u^k)$\;
    \STATE $\tau[y_u^k] \leftarrow \tau[y_u^k] + 1$\;
  \ENDFOR
  \FOR{$i=0$ to $L-1$}
    \STATE $\tau[i] \leftarrow \frac{\tau[i] - \gamma/L}{1-\gamma}$\;
  \ENDFOR
  \STATE \textbf{return} $\tau$
  \end{algorithmic}        
\end{algorithm}

\begin{thm}
\label{thm:complete_shuf}
Let $\eps \leq 1$, $\delta> 0 $ and $n \geq 14^2 \ln(4/\delta)$.  Algorithm~\ref{algo:completeshuff} with $\gamma = L/(e^{\eps}+L-1)$ achieves an unbiased estimate of the histogram with $\gamma T$ expected random responses. Furthermore, it satisfies $
(\eps',(N_v + \frac{T}{n}) \delta + \delta' + \hat{\delta})$-network
DP for all $\delta',\hat{\delta} > 0$ with 
\[ \eps'\leq \sqrt{\left(\frac{4T}{n} + 2\sqrt{\frac{3T}{n}\log(1/
\hat{\delta})}\right)\ln(1/\delta')} \frac{21 \sqrt{\ln(4/\delta)}}{\sqrt{n}}\eps + \sqrt{\frac{2T}{n} + \sqrt{\frac{3T}{n}\log(1/\hat{\delta})}}  \eps (e^{21 \eps \sqrt{\ln(4/\delta)}/\sqrt{n}}) - 1).\]
\end{thm}
\begin{rmk}
In the proof below, we use some approximations to obtain the simple
closed-form expressions of Theorem~\ref{thm:complete_shuf}. These
approximations however lead to
the
unnecessarily strong
condition $n \geq 14^2 \ln(4/\delta)$ and suboptimal constants in $\eps'$.
In concrete
implementations, we can obtain tighter results by numerically evaluating the complete formulas.
\end{rmk}

% \begin{rmk}[Choice of amplification by shuffling bound]
% We choose here to use the bound of \cite{clones} instead of 
% \cite{amp_shuffling} as it is more tight and holds under less
% restrictive assumptions.
% \end{rmk}

\begin{proof}
The proof follows the same steps as in the case of real summation 
(Appendix~\ref{completeproof}),
using the same ``fictive'' walk trick. We only need to adapt how we bound the
privacy loss of a given cycle. More precisely, keeping the same
notations as in Appendix~\ref{completeproof}, we need to modify how we bound
the modification of the
privacy loss of $\mathcal{A}$. Here, $\A$ corresponds to the aggregation of
some
discrete contributions, which is equivalent to shuffling these contributions.
We can therefore rely on privacy amplification by shuffling.
Specifically here,
we use the bound of \citet[][Theorem
3.1 therein]{clones} which is more tight and holds under less restrictive assumptions
than the result of \cite{amp_shuffling}. We recall the result below below. 

\begin{thm}[Amplification by shuffling, \citealp{clones}]
\label{thm:shufflclone}
 For any data domain $\mathcal{X}$, let $\mathcal{R}^{(i)}: \mathcal{S}^{(1)}
 \times \cdots \times \mathcal{S}^{(i-1)} \times \mathcal{X} \rightarrow \mathcal{S}^{(i)}$ for $i \in[n]$ (where $\mathcal{S}^{(i)}$ is the range
space of $\mathcal{R}^{(i)}$) be a sequence of algorithms such that $
\mathcal{R}^{(i)}\left(z_1,\dots,z_{i-1}, \cdot\right)$ is an $\eps_{0}$-DP
local
randomizer for all values of auxiliary inputs $(z_1,\dots,z_{i-1}) \in \mathcal{S}^{
(1)} \times \cdots \times \mathcal{S}^{(i-1)}$. Let $\mathcal{A}_{\mathrm{s}}:
\mathcal{X}^{n} \rightarrow \mathcal{S}^{(1)} \times \cdots \times 
\mathcal{S}^{(n)}$ be the algorithm which takes as input a dataset
$(x_1,\dots,x_n) \in \mathcal{X}^{n}$, samples a uniform random permutation
$\pi$ over $[n]$, then sequentially computes $z_{i}=\mathcal{R}^{(i)}\left
(z_1,\dots, z_{i-1}, x_{\pi(i)}\right)$ for $i \in[n]$ and outputs $
(z_1,\dots,z_n)$. Then for any $\delta \in[0,1]$ such that $\eps_0 \leq \log
\left(\frac{n}{16 \log (2 / \delta)}\right), \mathcal{A}_{\mathrm{s}}$
satisfies $
(\eps_{shuff}, \delta)$-DP with
\[
\eps_{shuff} \leq \ln \left(1+\frac{e^{\eps_0}-1}{e^{\eps_0}+1}\left(\frac{8 
\sqrt{e^{\eps_0} \ln (4 / \delta)}}{\sqrt{n}}+\frac{8 e^{\eps_0}}
{n}\right)\right).
\]
\end{thm}

For clarity, we propose to use a simpler expression for $\epsilon_{shuff}$ (Eq.~
\ref{eq:simple_shuf} below)
which makes
the
asymptotic amplification in $O(1/\sqrt{n})$ explicit. However, it is
possible to keep the initial form of Theorem~\ref{thm:shufflclone} for numerical applications.
To derive a less tight but more tractable bound, we use the fact that $
\frac{e^x-1}{e^x +1} \leq \frac{x}{2}$, which gives:
\[ \eps_{shuff} \leq \left(1+\frac{\eps_0}{2}\left(\frac{8 \sqrt{e^{\eps_0}
\ln (4 / \delta)}}{\sqrt{n}}+\frac{8 e^{\eps_0}}{n}\right)\right). \]

We then use the hypothesis $\eps_0 \leq 1$ and the concavity of the logarithm
to obtain the following simple bound:
\begin{equation}
\label{eq:simple_shuf}
\eps_{shuff} \leq \frac{14 \sqrt{\ln(4/\delta)}}{\sqrt{n}} \eps_0.
\end{equation}

Here, contrary to the case of real summation, amplification by shuffling is
effective only for cycles whose length $m$ is large enough. To
mitigate this issue,
we remark that, since the $k$-ary randomized response protocol $\mathcal{A}$
satisfies $\eps$-LDP, we can always bound the privacy loss of $\mathcal{A}$ by
the local guarantee $\eps$.

Let us assume that $m \geq 14^2 \ln(4/ \delta)$. This implies that
$\frac{14 \sqrt{\ln(4/\delta)}}{\sqrt{m}} \eps \leq 1$. This inequality is the hypothesis needed to simplify the expression of the privacy loss with the amplification by subsampling, as in the proof of real summation:
\[
\log(1+ (1 - (1-1/n)^m)(e^{\frac{14 \sqrt{\ln(4/\delta)}}{\sqrt{m}} \eps}-1) \leq \frac{21 \sqrt{\ln(4/\delta) m}}{n}\eps.
\]

In particular, for every cycle, 
\[
\eps_{cycle} \leq \min\Big(\frac{3 m \eps}{2n}, \frac{21 \sqrt{\ln(4/\delta)
m}}{n}\eps\Big),
\]
where the first term corresponds to the analysis where we use
amplification by subsampling and $\eps$ for the privacy
loss of $\mathcal{A}$,
while the second one is obtained by combining amplification by subsampling
with amplification by shuffling using \eqref{eq:simple_shuf} in the case of $m \geq 14^2 \ln(4/ \delta)$. We note
that the second term becomes smaller when $m$ is larger than $m = 14^2 \ln
(4/\delta)$. In this regime, the constraint $\eps \leq \ln(\frac{m}{16 \ln
(2/\delta)})$ required by Theorem 
\ref{thm:shufflclone} is directly satisfied, as $\eps \leq \ln(
\frac{14^2 \ln(4/\delta)}{16 \ln(2/\delta)})$ is less restrictive than $\eps
\leq 12.25$. As we assume that $n \geq 14^2 \ln(4/\delta)$, the regime where
the second term is larger exists. We see that the worst privacy loss is
reached for a cycle of length $n$, for which we have:
\[
\eps_{cycle} \leq \frac{21 \sqrt{\ln(4/\delta)}}{\sqrt{n}}\eps.
\]

Using the above bound for the privacy loss of any cycle, we conclude by
applying advanced composition as in the case of real aggregation.
\end{proof}

\section{Proof of Theorem~\ref{thm:iteration} (Stochastic Gradient Descent on
a Complete Graph)}
\label{sgdproof}

\begin{proof}

The proof tracks privacy loss using Rényi
Differential Privacy (RDP)
\citep{RDP} and leverages results on amplification by iteration 
\citep{amp_iter}. We first recall the definition of RDP and the main theorems
that we will use. Then, we apply these tools to our setting and conclude by
translating the resulting RDP bounds into $(\eps, \delta)$-DP.

Rényi Differential Privacy quantifies the privacy
loss based on the
Rényi divergence between the outputs of the algorithm on neighboring
databases.

\begin{defn}[Rényi divergence]
Let $1<\alpha<\infty$ and $\mu, \nu$ be measures such that for all measurable set $A$, $\mu(A)=0$ implies $\nu(A)=0$. The Rényi divergence of order $\alpha$ between $\mu$ and $\nu$ is defined as
\[
D_{\alpha}(\mu \| \nu) = \frac{1}{\alpha-1} \ln \int\left(\frac{\mu(z)}{\nu
(z)}\right)^{\alpha} \nu(z) d z.
\]
\end{defn}

In the following, when $U$ and $V$ are sampled from $\mu$ and $\nu$
respectively, with a slight abuse of notation we will often write $D_\alpha(U ||
V)$ to mean $D_{\alpha}(\mu \| \nu)$.

\begin{defn}[Rényi DP]
For $1 < \alpha \leq \infty$ and $\eps \geq 0$, a randomized algorithm $
\mathcal{A}$ satisfies $(\alpha, \eps)$-Rényi differential privacy, or $
(\alpha, \eps)$-RDP, if for all neighboring data sets $D$ and $D'$ we
have
\[
D_{\alpha}\left(\mathcal{A}(D) \| \mathcal{A}\left(D'\right)\right)
\leq \eps.
\]
\end{defn}

We can introduce a notion of \emph{Network-RDP} accordingly.

\begin{defn}[Network Rényi DP]
For $1 < \alpha \leq \infty$ and $\eps \geq 0$, a randomized algorithm $
\mathcal{A}$ satisfies $(\alpha, \eps)$-network Rényi differential privacy,
or
$(\alpha, \eps)$-NRDP, if for all pairs of
  distinct users $u, v\in V$ and all
  pairs   of neighboring datasets $D \sim_u D'$, we have
\[
D_{\alpha}\left(\Obs{v}(\mathcal{A}(D)) \| \Obs{v}(\mathcal{A}\left(D')\right)\right)
\leq \eps.
\]
\end{defn}

As in classic DP, there exists
composition theorems for RDP, see \cite{RDP}.
We will use the following.

\begin{prop}[Composition of RDP]
If $\mathcal{A}_{1}, \ldots, \mathcal{A}_{k}$ are randomized algorithms
satisfying $(\alpha, \eps_{1})\text{-RDP}, \ldots,(\alpha, \eps_{k})$-RDP
respectively, then their composition $(\mathcal{A}_{1}(S), \ldots, 
\mathcal{A}_{k}(S))$ satisfies $(\alpha, \sum_{l=1}^k \eps_l)$-RDP. Each
algorithm can be chosen adaptively, i.e., based on the outputs of
algorithms that come before it.
\end{prop}

Finally, we can translate the result of the RDP by using the following
result \citep{RDP}.

\begin{prop}[Conversion from RDP to DP]
\label{prop:convert}
If $\mathcal{A}$ satisfies $(\alpha, \eps)$-Rényi differential privacy, then for all $\delta \in(0,1)$ it also satisfies $\left(\eps+\frac{\ln (1 / \delta)}{\alpha-1}, \delta\right)$ differential privacy.
\end{prop}

Privacy amplification by iteration \citep{amp_iter} captures the fact that for
algorithms that consist of \emph{iterative contractive updates}, not releasing
the intermediate results improve the privacy guarantees for the final result.
An important application of this framework is Projected Noisy Stochastic
Gradient Descent (PNSGD) in the centralized setting, where the trusted curator
only reveals the final model. More precisely, when iteratively updating a
model with PNSGD, any given step is hidden by subsequent
steps (the more subsequent steps, the better the
privacy). The
following result from \cite{amp_iter} (Theorem 23 therein) formalizes this.
\begin{thm}[Rényi differential privacy of PNSGD]
\label{thm:ampiter}
Let $\mathcal{W} \in \R^d$ be a convex set, $\mathcal{X}$ be an abstract data
domain and $\set{f'; x)}_{x \in
\mathcal{X}}$ be a family of convex $L$-Lipschitz and $\beta$-smooth function
over $\mathcal{K}$. Let PNSGD$(D, w_0, \eta, \sigma)$ be the algorithm that
returns $w_n\in\mathcal{W}$ computed recursively from $w_0\in\mathcal{W}$
using dataset $D=\{x_1,\dots,x_n\}$ as:
$$w_{t+1} = \Pi_{\mathcal{W}}(w_t - \eta (\nabla f(w_t;
    x_{t+1}) +
    Z)),\quad\text{where }Z\sim \mathcal{N}(0, \sigma^2I_d).$$
Then for any $\eta \leq
2/\beta, \sigma >0, \alpha > 1, t \in [n]$, starting point $w_0 \in 
\mathcal{K}$ and $D \in \mathcal{X}^n$, PNSGD satisfies $(\alpha, \frac{\alpha
\eps}{n+1-t})$-RDP for
its $t$-th input, where $\eps = \frac{2L^2}{\sigma^2}$ .
\end{thm}

In our context, we aim to leverage this result to capture the
privacy amplification provided by the fact that a given user $v$ will only
observe
information about the update of another user $u$
after some steps of the random walk. To account for the fact that this number
of steps will itself be random, we will use the so-called weak convexity
property of the Rényi divergence \citep{amp_iter}.

\begin{prop}[Weak convexity of Rényi divergence]
\label{prop:convexity}
Let $\mu_{1}, \ldots, \mu_{m}$ and $\nu_{1}, \ldots, \nu_{m}$ be probability
distributions over some domain $\mathcal{Z}$ such that for all $i \in[m], D_
{\alpha}\left(\mu_{i} \| \nu_{i}\right) \leq c /(\alpha-1)$ for some $c \in 
(0,1]$. Let $\rho$ be a probability distribution over $[m]$ and denote by
$\mu_{\rho}$ (resp. $\nu_{\rho}$) the probability distribution over $
\mathcal{Z}$ obtained by sampling i from $\rho$ and then outputting a random
sample from $\mu_{i}$ (resp. $\nu_{i}$). Then we have:
\[
D_{\alpha}\left(\mu_{\rho} \| \nu_{\rho}\right) \leq(1+c) \cdot \underset{i \sim \rho}{\E}\left[D_{\alpha}\left(\mu_{i} \| \nu_{i}\right)\right].
\]
\end{prop}

We now have all the technical tools needed to prove our result.
Let us denote by $\sigma^2=\frac{8L^2\ln(1.25/\delta)}{\eps^2}$ the variance
of the Gaussian noise added at each gradient step in Algorithm~\ref{algo:sgd}.
Let us fix two distinct users $u$ and $v$. We aim to quantify how much
information about the private data of user $u$ is leaked to $v$ from
the visits of the token. In contrast to the proofs of
Theorem~\ref{thm:complete} (real summation) and
Theorem~\ref{thm:complete_shuf} (discrete histogram computation), we
will reason here on
the privacy loss induced by each contribution of user
$u$, rather than by each visit of the token through $v$.
% In particular, it is
% enough to consider the privacy loss of a single contribution.

Let us fix a contribution of user $u$ at some time $t_1$. The view $\Obs{v}$
of
user $v$ on the entire procedure is defined as in the proof of
Theorem~\ref{thm:complete}.
Note that the token values observed before $t_1$ do not depend on the
contribution of $u$ at time $t_1$. Let $t_2>t_1$ be the first time that $v$
receives the token posterior to $t_1$. It is sufficient to bound the privacy
loss induced by the observation of the token at $t_2$: indeed, by the
post-processing property of DP, no additional privacy loss with respect to
$v$ will occur for observations posterior to $t_2$.

By definition of the random walk, $t_2$ follows a geometric law of parameter
$1/n$, where $n$ is the number of users. Additionally, if there is no time
$t_2$ (which can be seen as $t_2>T$), then no privacy loss occurs. Let $Y_v$
and
$Y_v$ be the distribution followed by the token when observed by $v$ at
time $t_2$
for two neighboring datasets $D \sim_u D'$ which only differ in the dataset of
user $u$. For any $t$, let also $X_t$ and $X'_t$ be the
distribution followed by the token at time $t$ for two neighboring datasets $D \sim_u D'$.
Then, we can apply Proposition~\ref{prop:convexity} to $D_{\alpha}(Y_v||
Y'_v)$ with $c=1$, which is
ensured
when $\sigma \geq L\sqrt{2\alpha (\alpha - 1)}$, and we have:

\[ D_{\alpha}(Y_v|| Y'_v) \leq (1+1) \E_{t \sim \mathcal{G}(1/n)} D_{\alpha}(X_t || X'_t). \]

We can now bound $D_ {\alpha}(X_t || X'_t)$ for
each $t$
using
Theorem~\ref{thm:ampiter} and obtain:
\[ \begin{array}{lll}
D_{\alpha}(Y_v|| Y'_v) & \leq & \sum_{t=1}^{T-t_1} \frac{1}{n} (1 - \frac{1}{n})^t \frac{2 \alpha L^2}{\sigma^2 t} \\
    & \leq & \frac{2 \alpha L^2}{\sigma^2 n} \sum_{t=1}^{\infty} \frac{(1- 1/n)^t}{t}\\
    & \leq & \frac{2 \alpha L^2 \ln n}{\sigma^2 n}.
\end{array}
\]

To bound the privacy loss over all the $T_u$ contributions of user $u$, we use
the
composition property of RDP, leading to the following Network RDP guarantee.

% We finally use the composition to take into account the several contributions
% of node $u$. Assuming that  bounding the number of contributions by $N_u$,
% we have proven the following theorem:

\begin{lem}
Let $\alpha > 1$, $\sigma \geq L \sqrt{2 \alpha (\alpha - 1)}$ and $T_u$ be
maximum number of contributions of a user. Then
Algorithm~\ref{algo:complete} satisfies $(\alpha, \frac{4 T_u \alpha L^2 \ln
n}
{\sigma^2n})$-Network Rényi DP.
\end{lem}

We can now convert this result into an $(\eps_{\text{c}}, \delta_{\text{c}})$-DP
statement using
Proposition~\ref{prop:convert}. This proposition calls for minimizing the
function $\alpha \rightarrow \eps_{\text{c}}(\alpha)$
for $\alpha\in(1,\infty)$. However, recall that from our use of the weak convexity
property we have the additional constraint
on
$\alpha$ requiring that $\sigma
\geq L\sqrt{2 \alpha (\alpha - 1)}$. This creates two regimes: for
small $\eps_{\text{c}}$ (i.e, large $\sigma$ and small $T_u$), the minimum
is not reachable, so we take the best possible $\alpha$ within the interval,
whereas we have an optimal regime for larger $\eps_{\text{c}}$. This
minimization can be done numerically, but for simplicity of
exposition we can derive a suboptimal closed form which is the one given in
Theorem~\ref{thm:iteration}.

To obtain this closed form, we reuse the result of \cite{amp_iter} (Theorem
32 therein). In
particular, for $q = \max \big(\frac{2 T_u \ln n}
{n}, 2 \ln(1/\delta_{\text{c}}) \big)$, $\alpha = \frac{\sigma \sqrt{\ln(1/\delta_{\text{c}})} }{L 
\sqrt{q}}$ and $\eps_{\text{c}} = \frac{4L \sqrt{q \ln(1/\delta_{\text{c}})}}{\sigma}$, the
conditions $\sigma \geq L\sqrt{2 \alpha (\alpha - 1)}$ and $\alpha > 2$ are satisfied. Thus, we have a bound on the privacy loss which holds the two
regimes thanks to the definition of $q$.

Finally, we bound $T_u$ by $N_u = \frac{T}{n} + \sqrt{\frac{3T}{n} \log(1/
\hat{\delta})}$ with probability $1-\hat{\delta}$ as
done in the previous proofs for real summation and discrete histograms.
Setting $\eps'=\eps_{\text{c}}$ and $\delta'=\delta_{\text{c}} + \hat{\delta}$
concludes the proof.
\end{proof}

\begin{rmk}[Tighter numerical bounds]
As mentioned in the proof, we can compute a tighter bound for small $\sigma$
when the optimal $\alpha$ violates the constraints on $\sigma$. In this case,
we set $\alpha$ to its limit such that $\sigma = L \sqrt{2 \alpha (\alpha-1)}$ and
deduce a translation into $(\eps_{\text{c}}, \delta_{
\text{c}})$-differential privacy. This is
useful when $q \neq \frac{2 N_u \ln n}{n}$, i.e., situations where the
number of contributions of a user is smaller than the number of users.

In particular, we use this method in our experiments of
Section~\ref{sec:exp_sgd}. In that
case, we have a
fixed $(\eps_{\text{c}},\delta_{\text{c}})$-DP constraint and want to
find the
minimum possible $\sigma$ that ensures this privacy guarantee. We start with
a small candidate for $\sigma$ and compute the associated privacy loss as
explained above. We then increase it iteratively until the resulting $\eps_{
\text{c}}$ is
small enough. \end{rmk}

\section{Lifting the Assumption of Hidden Sender/Receiver}
\label{spotted}

In the analysis of Section~\ref{sec:complete}, we assumed that a user does not
know the identity of the previous (sender) or next (receiver) user in the
walk. We discuss here how we can lift this assumption. Our approach is based
on separately bounding the privacy loss of contributions that are adjacent to
a given user (\emph{spotted
contributions}), as these contributions do not benefit from any privacy
amplification if the identity of the sender/receiver is known.
We first compute the privacy loss resulting from spotted contributions, then
discuss in which regimes this term becomes negligible in the total theoretical
privacy loss, and finally how to deal with it empirically.

\begin{defn}[Spotted contribution]
For a walk on the complete graph, we define a spotted
contribution of $u$ with respect to $v$ as a contribution of $u$ that is
directly preceded or followed by a contribution of $v$.
\end{defn}

A spotted contribution has a privacy loss bounded by $\eps$, as we still have
the privacy guarantee given by the local randomizer, but no further
amplification of privacy. Thus, we need to bound the number of contributions
for a given vertex $u$ to be spotted by another user $v$. As in the proofs of
Theorem~\ref{thm:complete}, Theorem~\ref{thm:iteration} and
Theorem~\ref{thm:complete_shuf}, we bound the number of contributions of $u$
by $N_u$ using Chernoff. Now, for each of these contributions,
the probability of being spotted is
$2/n$, so the number of spotted contributions follows a binomial law of
parameter $\mathcal{B}(N_u, 2/n)$. We then use once again Chernoff to bound
the number of spotted contributions with probability $\tilde{\delta}$, and use
either simple or advanced composition. This leads to the
following bound for the privacy loss associated with spotted contributions.

\begin{lem}[Privacy loss of spotted contributions]
For a random walk with $N_u$ contributions of user $u$, the privacy loss
induced by spotted contributions is bounded with probability $1 -
\tilde{\delta}$ by:
\begin{itemize}
\item 
% \begin{multline}
$\eps_s =  
\sqrt{\left(\frac{2N_u}{n}+ \sqrt{\frac{6N_u}{n} \log(1/\tilde{\delta})}\right)\log(1/\delta')}\eps + 
\left(\frac{2N_u}{n}+ \sqrt{\frac{6N_u}{n} \log(1/\tilde{\delta})}\right) \eps
(e^{\eps} - 1)$ with advanced composition,
\item $\eps_s = 
\left(\frac{2N_u}{n}+ \sqrt{\frac{6N_u}{n} \log(1/\tilde{\delta})}\right)\eps$
with simple composition.
% \end{multline}
\end{itemize}
\end{lem}

The above term (along with $\tilde{\delta}$) should be added to the total
privacy loss to take into account the knowledge of the previous and next user.
The difficulty comes from the fact that the number of spotted contributions
has a high variance if the number of contributions per user is small compared
to the number of users. We already observed this when bounding the
number of contributions per user, where the worst case is far from the
expected value
(see Figure~\ref{fig:fixed_contrib}). Here, the price to pay
is higher as the square root dominates the expression in the regime
where $T=o(n^2)$. However, for $T=\Omega(n^2)$, the spotted contribution
term becomes negligible and we recover the same order of privacy amplification
as in Theorem~\ref{thm:complete}, Theorem~\ref{thm:iteration} and
Theorem~\ref{thm:complete_shuf}.

% Indeed, when $N_v \leq n$, the term in $\sqrt{N_v/n}$ dominates and the constant is a bit larger than the other terms until around $n \sim 10^6$ for classic choice of parameters.

The derivations above provide a way to bound the impact of spotted
contributions theoretically, but we can also deal with it
empirically. In practical implementations, one can also enforce a bound on the
number of times that an edge can be used, and dismiss it afterwards, with
limited impact on the total privacy loss. Another option to keep the same
formal guarantees is to replace real contributions with only noise when an
edge has exceeded the bound. These ``fake contributions'' seldom happen in
practice and thus do not harm the convergence.

\input{exp_supp}

%% file: exp_supp.tex
% !TEX root = main_supp.tex

\begin{figure}[t] 
  % \begin{subfigure}[b]{0.5\linewidth}
  %   \vspace*{.3cm}
    \centering
    \includegraphics[width=.45\linewidth]{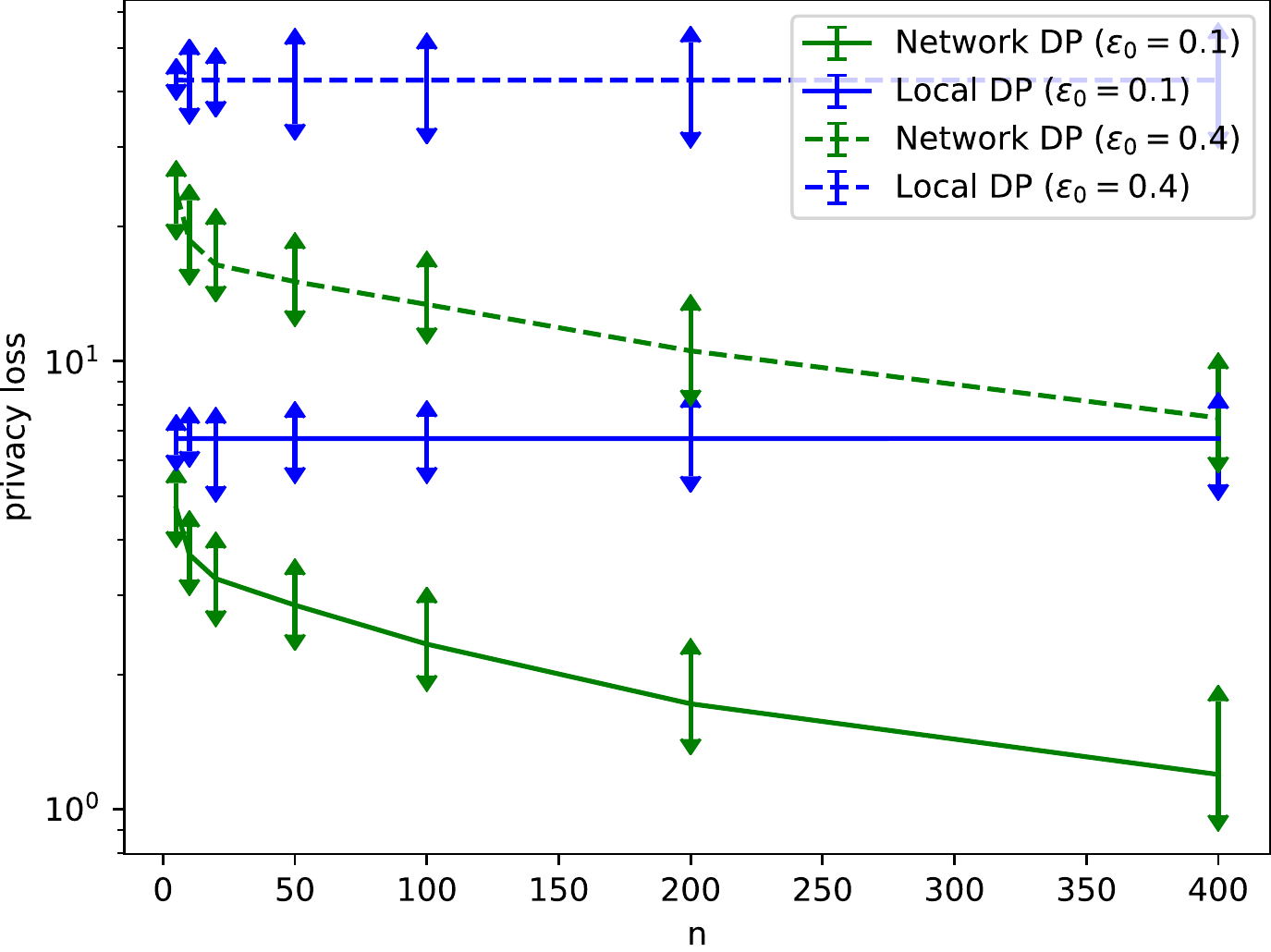}
    % \caption{Discrete histograms (empirical)\label{fig:shuffling}}
  % \end{subfigure} 
  % \begin{subfigure}[b]{0.5\linewidth}
  %   \centering
  %   \includegraphics[width=.8\linewidth]{fig/iterationboth.pdf}
  %   \caption{Stochastic gradient descent (empirical)\label{fig:iteration}}
  % \end{subfigure}
  \caption{Comparing network and local DP on the task of computing discrete
  histograms. The
  results are obtained for $T=100n$ (i.e., the expected number of
  contributions per user is $100$). The value of $\eps_0$
  rules the amount of local noise added to each contribution (i.e., each
  single contribution taken in isolation satisfied $\eps_0$-LDP). Curves
  report the average privacy
  loss across all pairs of users and all $10$ random runs, while their error
  bars give the best and worst cases.}
  \label{fig:shuffling}
\end{figure}

\section{Additional Experiments}
\label{exp_supp}

Similar to Section~\ref{sec:exp_sum}, we run experiments to
investigate the empirical behavior of our
approach for the task of discrete histogram computation on the complete
graph by leveraging results
on privacy amplification by shuffling. Here, we have used the numerical
approach from \cite{Balle2019} to tightly measure the effect of amplification
by shuffling based on the code provided by the authors.\footnote{
\url{https://github.com/BorjaBalle/amplification-by-shuffling}}
Figure~\ref{fig:shuffling}
% and \ref{fig:iteration}
confirm that
the empirical gains from privacy amplification by decentralization are also
significant for this task.